\documentclass[letterpaper,11pt,reqno]{amsart}

\makeatletter
\usepackage{amssymb}
\usepackage{latexsym}
\usepackage{amsbsy}
\usepackage{amsfonts}
\usepackage{hyperref}
\usepackage{graphicx}
\usepackage{scalerel}
\usepackage{enumerate}
\usepackage{enumitem}
\usepackage{mathtools}
\usepackage{color}
\usepackage{tcolorbox}
\usepackage{booktabs}
\usepackage{subcaption}

\usepackage{tikz}
\usetikzlibrary{arrows.meta, positioning}

\def\marginpar#1{\ignorespaces}

\DeclareMathOperator\argmax{\arg\max}

\newtheorem{theorem}{Theorem}[section]

\newtheorem{corollary}[theorem]{Corollary}
\newtheorem{definition}[theorem]{Definition}

\numberwithin{equation}{section}
\makeatother
\begin{document}
\title[CSDM]{Diffusion generative models meet compressed sensing,
\\with applications to imaging and finance}

\author[Zhengyi Guo]{{Zhengyi} Guo}
\address{Department of Industrial Engineering and Operations Research, Columbia University. 
} \email{zg2525@columbia.edu}

\author[Jiatu Li]{{Jiatu} Li}
\address{Department of Statistics, Columbia University. 
} \email{jl6969@columbia.edu}

\author[Wenpin Tang]{{Wenpin} Tang}
\address{Department of Industrial Engineering and Operations Research, Columbia University. 
} \email{wt2319@columbia.edu}

\author[David D. Yao]{{David D.} Yao}
\address{Department of Industrial Engineering and Operations Research, Columbia University. 
} \email{ddy1@columbia.edu}

\date{\today} 
\begin{abstract}
In this study we develop dimension-reduction techniques to accelerate diffusion model inference in the context of synthetic data generation.
The idea is to integrate compressed sensing into diffusion models (hence, CSDM):
First, compress the dataset into a latent space (from an ambient space),
and train a diffusion model in the latent space; next,
apply a compressed sensing algorithm to the samples generated in the latent space for decoding back to 
the original space; and 
the goal is to facilitate the efficiency of both model training and inference.
Under certain sparsity assumptions on data, 
our proposed approach achieves provably faster convergence, 
via combining diffusion model inference with sparse recovery.
It also sheds light on the best choice of the latent space dimension.
To illustrate the effectiveness of this approach, 
we run numerical experiments on a range of datasets,
including handwritten digits, medical and climate images, 
and financial time series for stress testing. 
Our code is available at \url{https://github.com/ZhengyiGuo2002/CSDM-code}.
\end{abstract}

\maketitle

\section{Introduction}
\label{sc1}

\quad Diffusion models have played a central role in 
the recent success in text-to-image creators 
such as DALL·E 2 \cite{Ramesh22} and Stable Diffusion \cite{Rombach22},
and in text-to-video generators such as Sora \cite{Sora}, Make-A-Video \cite{Singer22} and Veo \cite{Veo}.
Despite their success in the domain of computer vision
(and more recently in natural language processing \cite{Nie25, KK25}),
the usage of diffusion models for data generation in other fields such as operations research 
and operations management remains underdeveloped.
In those application domains, the diffusion models are prohibitively
demanding in computational effort for both training and inference, which will typically
require a large number of function 
evaluations (NFEs) in high-dimensional ambient spaces, creating bottlenecks
in major performance benchmarks such as
memory bandwidth and wall-clock time, rendering the models impractical for 
real-time and on-device deployment.

\quad As observed in \cite{Do06, Pop21, WZZC25}, many existing datasets enjoy low-dimensional structures.
So a natural solution to the difficulties mentioned above
is to apply dimension reduction techniques to diffusion models. 
The pioneer work \cite{Karras22, Rombach22} proposed the idea of training a diffusion model on a 
{\em latent} space instead of directly on the ambient space. 
This has triggered subsequent works on finding a suitable low-dimensional latent space for diffusion model training (see e.g., \cite{CHZW23, CX25}).
Also refer to \cite{MX25} for inference time scaling for diffusion models.

\quad The objective of our study here is also to accelerate diffusion generation
by exploiting the sparsity nature of the underlying dataset. 
Specifically, we develop an integrated compressed sensing and diffusion model (CSDM)
with the following features:
\begin{itemize}[itemsep = 3 pt]
\item 
We embed a sparse recovery algorithm in compressed sensing \cite{CT05, CT06, Do06} into the 
 diffusion model via the following steps, which we call the {\it CSDM (Generation) Pipeline}:
(i) compress the data in $\mathbb{R}^d$ into a low-dimensional latent space $\mathbb{R}^m$ ($m \ll d$);
(ii) train a diffusion model in the compressed/latent space $\mathbb{R}^m$ for inference;
(iii) apply the sparse recovery algorithm FISTA 
to the samples generated in the latent space for decoding back to $\mathbb{R}^d$. 
Refer to the flow diagram in the figure below.
\item
We provide a complexity analysis of CSDM that accounts for the computational efforts in both
the diffusion inference and the compressed sensing recovery. This leads to, as a byproduct, 
some useful guidance on the choice of
the latent space dimension.
(For instance, in the very sparse setting, 
the commonly adopted DDPM model \cite{Ho20} with FISTA \cite{BT092, BT09} for recovery
yields the complexity $\mathcal{O}(\sqrt{d})$; 
hence, the optimal compressed dimension $m = \mathcal{O}(\sqrt{d})$.)
\item 
We apply the proposed CSDM pipeline to various image datasets,
including MNIST (handwritten digits), OCTMNIST (medical),
and ERA5 Reanalysis (climate).
Furthermore, motivated by the idea of dimension reduction in
compressed sensing, we embed principle component analysis (PCA),
another dimension-reduction technique favored by portfolio
analyses, into the diffusion model for applications that 
involve financial time series used in stress tests
for identifying systemic risk.
In all these experiments, the CSDM pipeline has successfully preserved high sample fidelity,
while delivering substantial wall-clock speedups.
\end{itemize}

\begin{figure}[htbp]
\centering
\begin{tikzpicture}[
    node distance=2cm and 5cm,
    img/.style={draw, circle, minimum size=1.5cm, text width=1.2cm, fill=blue!30},
    img1/.style={draw, circle, minimum size=1.5cm, text width=1.2cm, fill=green!30},
    arrow/.style={->, thick, >=Stealth},
    label/.style={font=\footnotesize},
    every node/.style={align=center}
  ]
  \node[img] (x) {$p_{data}(\cdot)$};
  \node[label, below=0cm of x] 
    {\textbf{Input}\\$x \in \mathbb{R}^d$};
  \coordinate[right=6.5cm of x] (arrow1end);
  \draw[arrow] (x) -- (arrow1end) node[midway, above] {Compression (by sketching)};
  \node[img1, right=0cm of arrow1end] (y) {$\widetilde{p}_{\tiny \mbox{data}}(\cdot)$};
  \node[label, right=0.3cm of y] 
    {\textbf{Compressed}\\$y = Ax \in \mathbb{R}^m$};
  \coordinate[below=2.68cm of y] (arrow2end);
  \draw[arrow] (y) -- (arrow2end) node[midway, right] {Diffusion model \\ generation};
  \node[img1, below=0cm of arrow2end] (yh) {};
  \node[label, right=0.3cm of yh] 
    {\textbf{Generated}\\$\widetilde{y} \in \mathbb{R}^m$};
  \coordinate[left=6.5cm of yh] (arrow3end);
  \draw[arrow] (yh) -- (arrow3end) node[midway, below] {Compressed Sensing};
  \node[img, left=0cm of arrow3end] (xh) {};
  \node[label, below=0cm of xh] 
    {\textbf{Recovered}\\$\widetilde{x} \in \mathbb{R}^d$};
\end{tikzpicture}
\caption{CSDM Generation Pipeline.}
\end{figure}
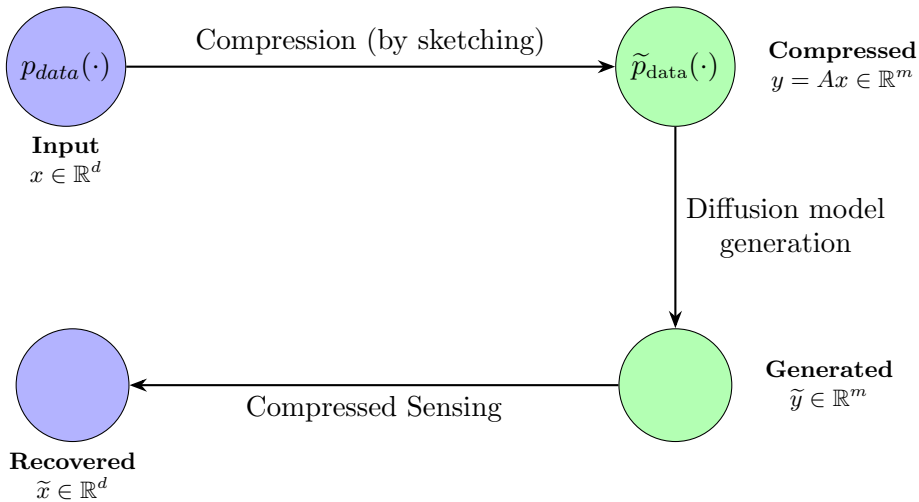

\quad 
We believe ours is the first study that formally integrates 
a whitebox encoder-decoder algorithm, 
such as FISTA in compressed sensing and PCA in 
financial analysis, into diffusion models, so that key components of both training and inference, such as
score evaluation, backpropagation and sampling, can benefit from the compressed dimension $m \ll d$, and 
achieve significantly improved efficiency and speedup.

\quad 
For applications, our approach is designed for decision-centric 
workflows that involve large scenario-based datasets, operating under a tight computing budget.
For example, in climate and energy applications, a common challenge is the capability to
generate compressed-domain ensembles of gridded weather fields (e.g., precipitation and irradiance),
and decode only the subsets needed for unit commitment, reserve sizing, and
chance-constrained optimal power flow. 
These are all essential and critical components in order to carry out
focused Monte-Carlo and what-if studies, while staying within a common wall-clock budget.

\quad 
The CSDM approach can be readily extended to other AI applications of diffusion models, 
such as fine-tuning/post-training/alignment \cite{Black24, Fan23, ZC25} 
(also see \cite[Section 4.5]{Win25} for a review). 
Notably, for instance, extending CSDM to fine-tuning will be similar 
in spirit to using 
a ``bad" (i.e., coarser) version of itself to achieve better results, as advocated in a recent work \cite{Karras24}.
Indeed, extensions in this direction will be the focus of our follow-up studies.

\medskip
{\bf Related Literature}:
Here we provide a brief review of the most relevant works.
Diffusion models were proposed by \cite{Ho20, SE19, Song20} 
in the context of generative modeling.
Empirically diffusion models have been shown to 
outperform other generative models such as GANs on various synthetic tasks \cite{Dh21, Kong21}.
Subsequent works studied the convergence of diffusion models \cite{Chen23, GNZ23, LW23}; 
see Section \ref{sc21} for more references.
As mentioned earlier, the training of diffusion models often suffers from the curse of dimensionality.
This leads to the works of finding provably good latent spaces for diffusion model training \cite{CHZW23, CX25}.

\quad 
There are also numerous approaches aiming at accelerating diffusion model inference,
including deterministic sampling \cite{SME21},  
higher-order ODE solvers \cite{Lu22, ZC23}, 
and progressive or consistency distillation \cite{SaH22, SD24, SD23}.
These sampling methods can be applied to diffusion inference in the latent space,
as such, they can be readily integrated into our CSDM framework.

\quad 
Recently, a line of theoretical studies \cite{HWC24, LY24, Pot24} 
explored the diffusion model's capability of adapting to low dimensionality, i.e.,
a diffusion model itself can capture the dataset's low-dimensional structure,
leading to faster convergence, without any dimension-reduction tricks.
Yet, these studies still require model training in the ambient space.
In contrast, the CSDM pipeline proposed here trains the model in the latent space, 
leading to more efficient training;
refer to Section \ref{sc3} for detailed analyses and further discussions.

\quad
Finally, it is worth noting that there are papers in the literature \cite{bora2017compressed, XC24} that apply 
generative models to help efficiently solve the inverse problems that are central 
to compressed sensing.
Our CSDM approach works in the {\it opposite} direction --
making compressed sensing {help} accelerate generation and inference in diffusion models.

\medskip
{\bf Organization of the paper}:
The rest of the paper is organized as follows.
Section \ref{sc2} highlights the background on diffusion models and preliminaries in compressed sensing.
The CSDM approach and its underlying theory are developed in Section \ref{sc3}.
Numerical experiments involving images and financial time series are reported, resepectively,
in Section \ref{sc4} and Section \ref{sc5}.
Concluding remarks are summarized in Section \ref{sc6}.

\section{Preliminaries}
\label{sc2}

\quad This section provides background materials on the two key subjects of the paper,
diffusion models and compressed sensing.

\quad Below we start with highlighting some symbols and notation that will be used throughout this paper.
\begin{itemize}[itemsep = 3 pt]
\item
For $x,y \in \mathbb{R}^d$, $x \cdot y$ denotes the inner product between $x$ and $y$,
and $|x|_p: = (\sum_{i = 1}^d  |x_i|^p)^{1/p}$ is the $p$-norm of $x$.
\item
For a function $f: \mathbb{R}^d \to \mathbb{R}$, let $\nabla f$ denote the gradient of $f$. 
\item
The symbol $\mathcal{N}(\mu, \Sigma)$ denotes the Gaussian distribution with mean $\mu$ and covariance matrix $\Sigma$,
and $\mbox{Unif}\, [a,b]$ denotes the uniform distribution on $[a,b]$.
\item
For $f: \mathbb{R}^d \to \mathbb{R}^m$ and $\mu(\cdot)$ a probability measure on $\mathbb{R}^d$,
the symbol $f_{\#}\mu(\cdot)$ denotes the pushforward of $\mu(\cdot)$ by $f$.
\item
The symbol $a = \mathcal{O}(b)$ or $a \lesssim b$ means that $a/b$ is bounded  as some problem parameter tends to $0$ or $\infty$ (often neglecting the logarithmic factor). 
\end{itemize}

\subsection{Diffusion models}
\label{sc21}

Diffusion models are a class of generative models that learn data distributions by a two-stage procedure:
the {\em forward process} gradually adding noise to data,
and the {\em reversed process} recovering/generating the data distribution $p_{\tiny \mbox{data}}(\cdot)$ from noise.
There are many formulations of diffusion models, 
e.g., by Markov chains \cite{Ho20, SE19},
by stochastic differential equations (SDEs) \cite{Song20},
and by deterministic flows \cite{Lip23, Liu22}.
To provide context, we briefly review the continuous-time formulation by SDEs
that offers a unified framework of diffusion models.

\quad We follow the presentation of \cite{TZ24}.
The forward process is governed by an SDE:
\begin{equation}
\label{eq:forward}
dX_t = f(t, X_t) dt + g(t) dW_t, \quad X_0 \sim p_{\tiny \mbox{data}}(\cdot),
\end{equation}
where $f: \mathbb{R}_+ \times \mathbb{R}^d \to \mathbb{R}^d$, $g: \mathbb{R}_+ \to \mathbb{R}_+$,
and $(W_t)_{t\ge 0}$ is Brownian motion in $\mathbb{R}^d$.
Some conditions are required on $f(\cdot, \cdot)$ and $g(\cdot)$ so that 
the SDE \eqref{eq:forward} is well-defined,
and that $X_t$ has a smooth density $p(t, x):= \mathbb{P}(X_t \in dx)/dx$, 
see \cite{SV79}.
As a specific and notable example, 
 $f(t,x) = -\frac{1}{2}(at + b)x$ and $g(t) = \sqrt{at + b}$ for some $a, b > 0$
 corresponds to the {\em variance preserving} (VP) model \cite{Song20},
 whose discretization yields the most widely used 
 {\em denoising diffusion probabilistic models} (DDPMs) \cite{Ho20}.

\quad The key to the success of diffusion models is that their time reversal $(\widetilde{X}_t)_{0 \le t \le T}$ has a tractable form:
\begin{equation*}
d\widetilde{X}_t = \left(-f(T-t, \widetilde{X}_t) + g^2(T-t) \nabla \log p(T-t, \widetilde{X}_t)\right) dt + g(T-t) dB_t, \quad \widetilde{X}_0 \sim p(T, \cdot),
\end{equation*}
with $(B_t)_{t \ge 0}$ a copy of Brownian motion in $\mathbb{R}^d$ \cite{HP86}.
It is common to replace $p(T, \cdot)$ with a noise  $p_{\tiny \mbox{noise}}(\cdot)$,
which is close to $p(T, \cdot)$ but should {\em not} depend on $p_{\tiny \mbox{data}}(\cdot)$.
All but the term $\nabla \log p(T-t, \widetilde{X}_t)$ are available,
so it comes down to learning $\nabla \log p(t,x)$,
known as {\em Stein's score function}.
Recently developed score-based methods attempt to approximate $\nabla \log p(t,x)$ by 
neural nets $\{s_\theta(t,x)\}_\theta$,
called {\em score matching}.
The resulting reversed process $(Y_t)_{0 \le t \le T}$ is:
\begin{equation}
\label{eq:reverse}
dY_t = \left(-f(T-t, Y_t) + g^2(T-t) s_\theta(T-t, Y_t) \right)dt + g(T-t) dB_t, \quad Y_0 \sim p_{\tiny \mbox{noise}}(\cdot).
\end{equation}
An equivalent (probabilistic) ODE sampler is:
\begin{equation}
\label{eq:reverse2}
dY_t = \left(-f(T-t, Y_t) + \frac{1}{2}g^2(T-t) s_\theta(T-t, Y_t) \right)dt, \quad Y_0 \sim p_{\tiny \mbox{noise}}(\cdot).
\end{equation}
Both \eqref{eq:reverse} and \eqref{eq:reverse2} 
are referred to as the {\em inference processes},
and the implementation requires discretizing these processes.

\quad There are several existing score matching methods, among which the most widely used one is
{\em denoising score matching} (DSM) \cite{Vi11}:
\begin{equation}
\label{eq:DSMcond}
\min_\theta \mathbb{E}_{t \sim \tiny \mbox{Unif}\, [0,T]} \bigg\{\lambda_t \, \mathbb{E}_{X_0\sim p_{data}} \left[ \mathbb{E}_{p(t, \cdot | X_0)}\Big|s_{\theta}(t,X_t)- \nabla \log p(t,X_t | X_0)\Big|_2^2 \right] \bigg\},
\end{equation}
where $\lambda_t$ is a weight function.
The advantage of DSM is that most existing models (e.g., VP) are Gaussian processes of form
$X_t = \alpha_t X_0 + \sigma_t \varepsilon$, 
with  $\varepsilon \sim \mathcal{N}(0,I)$ independent of $X_0$.
By adopting a noise parameterization $\varepsilon_\theta(t, X_t) = -\sigma_t s_\theta(t,X_t)$,
DSM \eqref{eq:DSMcond} reduces to:
\begin{equation}
\label{eq:DSMnoise}
\min_\theta \mathbb{E}_{t \sim \tiny \mbox{Unif}\, [0,T]} \bigg[\frac{\lambda_t}{\sigma_t^2} \, \mathbb{E}_{X_0\sim p_{data}, \varepsilon \sim \mathcal{N}(0,I)}  \left|\varepsilon_{\theta}(t,\alpha_t X_0 + \sigma_t \varepsilon)- \varepsilon\right|_2^2 \bigg].
\end{equation}
Common choices for the weight function are $\lambda_t = \sigma_t^2$ \cite{Song20},
and $\lambda_t = - \sigma_t^2 \left(\log\frac{\alpha_t^2}{\sigma_t^2} \right)'$ \cite{Kingma2021} corresponding to 
the evidence lower bound.
For analytical studies, it is standard to assume a blackbox score matching error:
there is $\epsilon > 0$ such that
\begin{equation}
\label{eq:bberr}
\mathbb{E}_{X \sim p(t, \cdot)}|s_{\theta_*}(t,X) - \nabla \log p(t, X)|_2^2 < \epsilon^2,
\end{equation}
where $\theta_*$ is output from some score matching algorithm (e.g., DSM).
See also \cite{CHZW23, HR24, WHT24} for analysis of score matching errors
based on specific neural network structures. 

\quad It is expected that under suitably good score matching, 
the output $Y_T$ or its discretization of the models \eqref{eq:reverse} and \eqref{eq:reverse2}
is close to $p_{data}(\cdot)$.
To simplify the presentation,
we focus on the VP model.
We need the following result on the $W_2$ convergence of the model.
\begin{theorem}
\label{prop:complexity}
Let $(Y, \widetilde{Y})$ be defined on the same probability space such that $Y \sim p_{data}(\cdot)$,
and $Y'$ is distributed as the output of the VP model.
Assume that 
$p_{data}(\cdot)$ is strongly log-concave, the score $\nabla \log p(t,x)$ is Lipschitz,
and the score matching error \eqref{eq:bberr} holds.
Then:
\begin{enumerate}[itemsep = 3 pt]
\item \cite{GNZ23}
There is a discretization of \eqref{eq:reverse} such that it takes $n_{\tiny \mbox{diff}} = \mathcal{O}(\frac{d}{\epsilon^2})$ steps to 
achieve $|Y - \widetilde{Y}|_2 \le \epsilon$ with high probability.
\item \cite{GZ25}
There is a discretization of \eqref{eq:reverse2} such that it takes $n_{\tiny \mbox{diff}} = \mathcal{O}(\frac{\sqrt{d}}{\epsilon})$ steps to achieve $|Y - \widetilde{Y}|_2 \le \epsilon$ with high probability.
\end{enumerate}
\end{theorem}

\quad The $W_2$ convergence of other diffusion models, e.g., variance exploding (VE) \cite{Song20},
was also studied in \cite{GNZ23, GZ25, TZ24b}.
See also \cite{Ben24, Chen23b, Chen23, LLT22, LH24, LW23, LY25} for the KL convergence 
under similar assumptions as in Proposition \ref{prop:complexity}.
In another direction, 
\cite{HWC24, LY24, Pot24} explored the adaptivity of diffusion models to (unknown) low dimensionality.
They showed that it takes $\mathcal{O}(\frac{d_{\tiny \mbox{IS}}}{\epsilon^2})$ steps,
where $d_{\tiny \mbox{IS}}$ is the {\em intrinsic dimension}, 
for DDPM to achieve an $\epsilon$ KL-error.
We defer the discussion to Section \ref{sc3}.

\subsection{Compressed sensing}
\label{sc22}

Compressed sensing \cite{CRT06, CT05, CT06, Do06} offers a powerful framework
for the exact recovery of a {\em sparse} signal $x \in \mathbb{R}^d$ from a limited number of observations $y \in \mathbb{R}^m$
with $m \ll d$.
We start by reviewing compressed sensing, following the presentation of \cite{Candes06}.
 
 \smallskip
 {\em Sparse recovery problem}:
Let $x = (x^1, \ldots, x^d)$, 
and assume that its support $T:=\{i: x^i \ne 0\}$ has small cardinality. 
The primary goal is to solve:
\begin{equation}
\label{eq:L0}
\min |x|_{0} \quad \text{subject to} \quad Ax = y.
\end{equation}
Solving this problem is equivalent to finding sparse solutions to an underdetermined system of linear equations,
which is NP-hard \cite{Nat95}.
The key idea of compressed sensing relies on $L^1$ techniques;
that is to transform the problem \eqref{eq:L0} into a linear program:
\begin{equation}
\label{eq:L1}
\min |x|_{1} \quad \text{subject to} \quad Ax = y, 
\end{equation}
which is known as {\em basis pursuit} \cite{CDS98}.

\quad In our application, we do not have exact compressed data $y$.
Instead, we have synthetic generation $\widetilde{y}$
that can be viewed as a measurement with noise.
This scenario fits into {\em robust compressed sensing} \cite{CRT06b}:
$y = Ax + e$,
where $e$ is some unknown perturbation with $|e|_2 \le \sigma$ ($\sigma$ is known).
It is natural to consider the convex program:
\begin{equation}
\label{eq:robustCS}
\min |x|_{1} \quad \text{subject to} \quad |Ax - y|_2 \leq \sigma. 
\end{equation}
In fact, the solution to \eqref{eq:robustCS} recovers a sparse signal with an error at most of the noise level.
To state the result, we need the following notion.
\begin{definition} \cite{CT05}
Let $A$ be the matrix with the finite collection of vectors $(v_j)_{j \in J} \in \mathbb{R}^m$ as columns. 
For each $1 \leq S \leq |J|$, 
we define the $S$-restricted isometry constant $\delta_S$ to be the smallest quantity such that $A_T$ obeys
\begin{equation*}
(1 - \delta_S)|c|_2^2 \leq |A_T c|_2^2 \leq (1 + \delta_S)|c|_2^2,
\end{equation*}
for all subsets $T \subset J$ of cardinality at most $S$, and all real coefficients $(c_j)_{j \in T}$.
\end{definition}

\quad The numbers $\delta_S$ measure how close the vectors $v_j$ behave like an orthonormal system,
but only when restricting to sparse linear combinations involving no more than $S$ vectors. 
The following theorem concerns sparse recovery for robust compressed sensing.
\begin{theorem} \cite{CRT06b}
\label{thm:robustCS}
 Let $S$ be such that $\delta_{3S} + 3\delta_{4S} < 2$. 
 Then for any signal $x$ supported on $T$ with $|T| \leq S$ (referred to as $S$-sparse), and any perturbation $e$ with $|e|_2 \leq \sigma$,
\begin{equation*}
|x_{*} - x|_2 \leq C_S \, \sigma, 
\end{equation*}
where $x_{*}$ is the solution to the problem \eqref{eq:robustCS}, and the constant $C_S$ only depends on $\delta_{4S}$. 
\end{theorem}

 \smallskip
 {\em Sparse recovery optimization}:
 The task is to solve numerically the optimization problem \eqref{eq:robustCS}.
It is known that this problem can be recast into an unconstrained convex problem:
\begin{equation}
\label{eq:QPCS}
\min \frac{1}{2} |Ax - y|_2^2 + \lambda|x|_{1},
\end{equation}
where the relation between $\lambda$ and $\sigma$ is specified by the Pareto frontier \cite{VF08}.
The problem \eqref{eq:QPCS}, known as Lasso or image deblurring problem,  
can be solved by several iterative algorithms, some of them are presented in \cite{zhao2023surveynumericalalgorithmssolve}.
Here we focus on one of these algorithms, 
{\em Fast Iterative Shrinkage-Thresholding Algorithm} (FISTA) \cite{BT092, BT09}.

\quad In the sequel, we denote  $f(x) := \frac{1}{2}|Ax-y|_2^{2}$ and $g(x):=\lambda |x|_1$.
Note that $\nabla f(x) = A^T(Ax -y)$, 
so 
 \begin{equation}
  \label{eq:lipschitz}
|\nabla f(x) - \nabla f(x')|_{2} 
\leq L |x - x'|_2 \quad \mbox{for all } x,x' \in \mathbb{R}^d,
 \end{equation}
where $L:= \lambda_{\max}(A^T A)$ is the largest eigenvalue of $A^TA$.
Define
\begin{equation*}
Q_L(x,x') := f(x') + \nabla f(x') \cdot (x - x') + \frac{L}{2} |x-x'|_2^2 + g(x),
\end{equation*}
and
\begin{equation}
\begin{aligned}
p_L(x') &= \argmax_x Q_L(x,x') \\
& = \argmax_x \left\{ \frac{L}{2} \left| x - \left(x' - \frac{\nabla f(x')}{L} \right) \right|_2^2+ g(x)\right\} \\
& = \operatorname{SoftThreshold}\left(x' - \frac{\nabla f(x')}{L} , \frac{\lambda}{L} \right),
\end{aligned}
\end{equation}
where the soft-thresholding operator is applied coordinate-wise \cite{CD98, DDD04}:
\begin{equation*}
\operatorname{SoftThreshold}(x,a)_i:= 
\begin{cases}
x_i - a & \text{if } x_i > a,\\
0 & \text{if } |x_i| \le a, \\
x_i + a & \text{if } x_i < - a.
\end{cases}
\end{equation*}
FISTA is a proximal gradient method by incorporating the Nesterov acceleration.
\begin{tcolorbox}
\textbf{Fast Iterative Shrinkage-Thresholding Algorithm (FISTA)}

\textbf{Input:} L (Lipschitz constant of $\nabla f$).

\textbf{Step 0.} Take \( y_1  = x_0 \in \mathbb{R}^d \), \( t_1 = 1 \).

\textbf{Step k.}  Compute
\[
x_k = p_L(y_k),
\]
\[
t_{k+1} = \frac{1 + \sqrt{1 + 4t_k^2}}{2},
\]
\[
y_{k+1} = x_k + \left( \frac{t_k - 1}{t_{k+1}} \right) (x_k - x_{k-1}).
\]
\end{tcolorbox}

\quad The convergence result is as follows.
\begin{theorem} \cite{BT092, BN17}
\label{thm:FISTA}
Let $x_*$ be the solution to the problem \eqref{eq:QPCS},
and $\{x_k\}_{k \ge 0}$ the FISTA iterates.
We have for $k$ sufficiently large,
\begin{equation*}
F(x_k) - F(x^*) \leq \frac{CL}{k^2} \quad \mbox{and} \quad
|x_k - x_*|_2 \le \frac{C(L + |y|_2)}{k}, 
\end{equation*}
for some $C > 0$.
\end{theorem}

\quad To ensure that $|x_k - x_*|_2 \leq \epsilon$,
it requires the number of iterations
$n_{\tiny \mbox{CS}} = \mathcal{O}\left( \frac{L}{\epsilon}\right) =  \mathcal{O}(\frac{s^2_{\max}(A)}{\epsilon})$,
where $s_{\max}(A)$ is the largest singular value of $A$.
Also refer to \cite{JM15, TBZ16} for sharper convergence results of FISTA (but implicit in the dimension dependence),
and \cite{Al19, CL18} for variants of FISTA.

\section{Main results}
\label{sc3}

\quad In this section, we develop the methodology by combining diffusion models with compressed sensing for sparse signal/data generation,
and provide theoretical insights.
As mentioned in the introduction, 
the idea is to compress the data into a lower dimension space, 
where a diffusion model is employed to generate samples more efficiently.
Compressed sensing is then used to convert the generated samples in the latent space
to the original signal/data space.
Our algorithm is summarized as follows.
\begin{tcolorbox}
\textbf{Compressed Sensing + Diffusion models (CSDM)}

\textbf{Input:} $A \in \mathbb{R}^{m \times d}$ (sketch matrix, $m \ll d$).

\textbf{Step 1.} Apply linear sketch to compress the data $p_{\tiny \mbox{data}}(\cdot)$ in $\mathbb{R}^d$
into $\widetilde{p}_{\tiny \mbox{data}}(\cdot): = A_{\#}p_{\tiny \mbox{data}}(\cdot)$ in $\mathbb{R}^m$.

\textbf{Step 2.}  Train a diffusion model using the data points drawn from $\widetilde{p}_{\tiny \mbox{data}}(\cdot)$.

\textbf{Step 3.}  Apply FISTA to solve the problem \eqref{eq:QPCS},
with $y$ generated by the diffusion model trained in Step 2.
\end{tcolorbox}

\quad While CSDM can be applied to any target data,
it is mostly efficient for generating sparse data distribution in the regime of compressed sensing.
The following theorem provides a theoretical guarantee for the use of CSDM in data generation.
\begin{theorem}
\label{thm:main}
Let $(x, \widetilde{y})$ be defined on the same probability space such that
$x \sim p_{\tiny \mbox{data}}(\cdot)$,
and $\widetilde{y}$ is output by the diffusion model in Algorithm CSDM.
Assume that $|Ax - \widetilde{y}|_2 \le \sigma$ with high probability.
Also let the assumptions in Theorem \ref{thm:robustCS} hold
(i.e., $p_{\tiny \mbox{data}}(\cdot)$ enjoys $S$-sparsity and $A$ satisfies the restricted isometry property).
For $\{x_k\}_{k \ge 0}$ the FISTA iterates relative to $\widetilde{y}$,
we have with high probability, 
\begin{equation}
|x_k - x|_2 \le C\left(\sigma + \frac{s^2_{\max}(A) + \sqrt{S}}{k} \right), \quad \mbox{for } k \mbox{ sufficiently large},
\end{equation}
where $s_{\max}(A)$ is the largest singular value of $A$.
\end{theorem}
\begin{proof}
Let $x_*$ be the solution to the problem:
\begin{equation*}
\min |x|_{1} \quad \text{subject to} \quad |Ax - \widetilde{y}|_2 \leq \sigma. 
\end{equation*}
By Theorem \ref{thm:robustCS},
we have $|x - x_*|_2 \le C \sigma$ for some $C > 0$.
Further by Theorem \ref{thm:FISTA}, 
we have for $k$ sufficiently large,
\begin{equation*}
|x_k - x_*|_2  \le \frac{C(L+|\widetilde{y}|_2)}{k} \le \frac{C(L + \sigma + |Ax|_2)}{k}.
\end{equation*}
Under the assumption of Theorem \ref{thm:robustCS},
the term $|Ax|_2$ is of order $\mathcal{O}(\sqrt{S})$. 
Thus, we get $|x_k - x_*|_2  \le \frac{C'(L + \sigma + \sqrt{S})}{k}$ 
for some $C' > 0$ and for $k$ sufficiently large.
By triangle inequality, 
we have $|x_k - x|_2 \le |x_k - x_*|_2 + |x_* - x|_2$,
which yields the desired result.
\end{proof}

\quad Specializing to the VP model leads to the following corollary.
\begin{corollary}
\label{coro:main}
Let the assumptions in Theorem \ref{prop:complexity} and Theorem \ref{thm:main} hold,
with $\widetilde{y}$ be the output of the discretized VP model in $k'$ steps,
and $\{x_{k',k}\}_{k \ge 0}$ be the FISTA iterates as to $\widetilde{y}$.
Then:
\begin{enumerate}[itemsep = 3 pt]
\item
Using the stochastic sampler \eqref{eq:reverse},
we have for $k, k'$ sufficiently large,
\begin{equation}
\label{eq:stocbd}
|x_{k,k'} - x|_2 \le C\left(\sqrt{\frac{m}{k'}} + \frac{s^2_{\max}(A) + \sqrt{S}}{k}\right), \quad \mbox{for some } C > 0.
\end{equation}
\item
Using the deterministic sampler \eqref{eq:reverse2},
we have for $k, k'$ sufficiently large,
\begin{equation}
\label{eq:deterbd}
|x_{k,k'} - x|_2 \le C\left(\frac{\sqrt{m}}{k'} + \frac{s^2_{\max}(A)+ \sqrt{S}}{k}\right), \quad \mbox{for some } C > 0.
\end{equation}
\end{enumerate}
\end{corollary}

\quad Several remarks are in order:

(a) It is common to choose the sketch matrix $A \in \mathbb{R}^{m \times d}$ to be random,
e.g., each entry of $A$ is a Gaussian variable with mean $0$ and variance $\frac{1}{m}$.
By extreme value theory of random matrices \cite{RV10},
the largest singular value 
\begin{equation*}
s_{\max}(A) \lesssim \sqrt{\frac{d}{m}} \quad \mbox{with high probability}.
\end{equation*}
Thus, the Lipschitz constant $L = s^2_{\max}(A)$ is of order $\frac{d}{m}$.
Replacing $s^2_{\max}(A)$ with $\frac{d}{m}$ in \eqref{eq:stocbd}-\eqref{eq:deterbd} yields:
\begin{equation}
\label{eq:rdbd}
|x_{k,k'} - x|_2 \lesssim
\begin{cases}
\sqrt{\frac{m}{k'}} + \frac{1}{k} (\frac{d}{m} + \sqrt{S}) & \text{for the stochastic sampler},\\
\frac{\sqrt{m}}{k'} + \frac{1}{k} (\frac{d}{m} + \sqrt{S})   & \text{for the  deterministic sampler}.
\end{cases}
\end{equation}

(b) The two terms in the bounds \eqref{eq:stocbd}, \eqref{eq:deterbd} and \eqref{eq:rdbd}
correspond to the {\em diffusion sampling error} and the {\em compressed sensing optimization error}.
As mentioned in the introduction,
a tradeoff between these two errors leads to an optimal choice of $m$ -- the compressed data dimension.
Let's take the stochastic sampler of the VP model for example.
In order to get $|x_{k,k'} - x|_2 \le \epsilon$,
it requires:
\begin{equation*}
k' = \mathcal{O}\left( \frac{m}{\epsilon^2}\right) \quad \mbox{and} \quad
k = \mathcal{O}\left( \left(\frac{d}{m } + \sqrt{S}\right) \frac{1}{\epsilon}\right)
\end{equation*}
Also assume that in each iteration,
the computational cost of diffusion sampling is comparable to that of compressed sensing optimization \footnote{A subtlety is that score evaluations for diffusion inference are typically performed on modern GPUs,
while the sparse recovery via FISTA is usually conducted on CPUs. 
For most diffusion inference tasks, 
each denoising step takes $10$-$100$ ms.
On a CPU, an arithmetic operation takes typically $5$-$10$ ns,
and the values of $md$ range from $10^6$-$10^7$ in our experiments:
FISTA's per-iteration cost is of order $5$-$100$ ms. 
So it is reasonable to assume that the per-step cost in diffusion inference is comparable to FISTA's per-iteration cost.
The experiments in Section \ref{sc5} also show that the running time of FISTA is lightweight compared to the diffusion inference time.}.
Under this hypothesis, the complexity that combines sampling and optimization is of order:
\begin{equation}
\label{eq:compl}
\max\left(m, \frac{d}{m} + \sqrt{S}\right).
\end{equation}
Consider the very sparse case $S = \mathcal{O}(1)$.
Optimizing \eqref{eq:compl} with respect to $m$ yields  $m = \mathcal{O}(\sqrt{d})$,
with the resulting complexity $\mathcal{O}(\sqrt{d})$.
Similarly, for the deterministic sampler of the VP model, the optimal $m = \mathcal{O}(d^{\frac{2}{3}})$,
with the resulting complexity $\mathcal{O}(d^{\frac{1}{3}})$.

\smallskip
(c) Theorem \ref{thm:main} is flexible to support different sampling schemes and optimization algorithms.
Also assume that $S = \mathcal{O}(1)$.
Table \ref{table}  below summarizes the optimal $m$ and corresponding complexity 
under various sampling methods with FISTA for compressed sensing.

\begin{table}[h]
\begin{center}
    \begin{tabular}{| c | c | c | c | c |}
    \hline
    Sampling & VP (Deterministic) & VP (Stochastic)  & VE (Deterministic) & VE (Stochastic)  \\ \hline
    $m$ & $d^{\frac{2}{3}}$& $d^{\frac{1}{2}}$ & $d^{\frac{2}{5}}$ & $d^{\frac{2}{3}}$   \\ \hline
    Complexity  & $d^{\frac{1}{3}}$ & $d^{\frac{1}{2}}$ & $d^{\frac{3}{5}}$ & $d^{\frac{1}{3}}$  \\ \hline
    \end{tabular}\\
    
    \medskip
    \caption{Optimal $m$ and the corresponding complexity under different sampling schemes and FISTA for compressed sensing.}
    \label{table}
    \end{center}
 \end{table}
 
\quad There are also other (provable) optimization algorithms for solving compressed sensing \eqref{eq:QPCS}.
For instance, {\em iteratively reweighted least squares} (IRLS) \cite{CWB08, RY08, GR02} 
was proved to achieve the computational complexity $\mathcal{O}\left(\frac{d}{\sqrt{m}}\right)$ \cite{Ku21}.
So for the VP model,
the stochastic sampler with IRLS for compressed sensing yields
the optimal $m = d^{\frac{2}{3}}$ and the complexity $\mathcal{O}(d^{\frac{1}{3}})$;
and the deterministic sampler with IRLS for compressed sensing yields
the optimal $m = \mathcal{O}(d)$ and the complexity $\mathcal{O}(d^{\frac{1}{2}})$.

\smallskip
(d) As mentioned earlier, 
a recent line of works \cite{HWC24, LY24, Pot24} 
studied the diffusion model's capability of adapting to low dimensionality.
It was shown that the complexity (in KL) of DDPM, 
a version of the stochastic sampler of the VP model, is
$\mathcal{O}(d_{\tiny \mbox{IS}})$,
where $d_{\tiny \mbox{IS}}$ is the intrinsic dimension
defined as the logarithm of the data's metric entropy.
Under the $S$-sparsity assumption, it is known \cite{Ver09} that
\begin{equation}
d_{\tiny \mbox{IS}}  = \mathcal{O}(S \log d).
\end{equation}
This yields the complexity $\mathcal{O}(S \log d)$ for DDPM in KL divergence.

\quad On the other hand,
it requires $m \gtrsim S$ to ensure that $A$ satisfies the $S$-restricted isometry property.
It then follows from \eqref{eq:compl}:
\begin{equation}
\mbox{the complexity of CSDM} = 
\begin{cases}
\mathcal{O}(S) & \text{if } S \gtrsim \sqrt{d},\\
\mathcal{O}(\sqrt{d}) & \text{if } S \lesssim \sqrt{d}.
\end{cases}
\end{equation}
If the results of \cite{HWC24, LY24, Pot24} also hold in $L^2$ norm,
then our proposed CSDM achieves the same complexity as theirs when $S \gtrsim \sqrt{d}$.
Note that sharper bounds on the FISTA convergence (e.g., independent of $S$)
will lead to a better complexity for CSDM than that of a diffusion model alone.
Moreover, diffusion models are typically easier to train in low-dimensional spaces than in high-dimensional settings.
We also mention the work \cite{CHZW23},
which proposed to project data onto a low-dimensional space for efficient score matching,
as opposed to direct generation.

\section{Numerical experiments on images}
\label{sc4}

\quad Here we conduct numerical experiments on various sparse image datasets, 
including handwritten digits (MNIST), medical images (OCTMNIST), and climate images (ERA5 Reanalysis).
Generating such data plays an important role in advancing further analytical methodologies across domains such as supply chain logistics, healthcare, and energy systems.

\quad Due to the inherently low resolution of publicly available datasets (e.g., MNIST and OCTMNIST), 
we adopt a {\em resolution upscaling strategy}:
all images are resized to larger spatial dimensions,
while preserving their inherent sparsity structure.
This ensures that the dimensionality is sufficiently high to corroborate our proposed CSDM framework. 
Upscaling is applied only to form the ambient dimension $d$ for time comparisons, 
whereas the compressed dimension $m$ is fixed across all $d$.
Take the MNIST dataset for instance:
we fix the sketch matrix $A \in \mathbb{R}^{m \times d}$ across all experiments for each $d \in \{32 \times 32, 40 \times 40, 48 \times 48\}$, where $m = 28^2 = 784$ is the compressed dimension.
This allows us to evaluate our method under varying degrees of compression,
corresponding to $77\%$, $49\%$, and $34\%$ respectively.

\quad In our proposed pipeline, the total generation time per image consists of:
\begin{itemize}[itemsep = 3 pt]
\item 
{\em Diffusion inference time} $T^{(m)}_{\text{diff}}$: 
the inference time of the diffusion model in $\mathbb{R}^m$. 
\item 
{\em Recovery time} $T^{(m,d)}_{\text{CS}}$: the time required for compressed sensing  ($\mathbb{R}^m \to  \mathbb{R}^d$) via FISTA.
\end{itemize}
So the total generation time of our algorithm is 
$T_{\text{total}} = T^{(m)}_{\text{diff}} + T^{(m,d)}_{\text{CS}}$.
Our goal is to measure the speedup over the baseline, 
which is to perform diffusion inference directly in $\mathbb{R}^d$.
Here we adopt the stochastic sampler,
so $T_{\text{diff}}^{(d)} \approx \frac{d}{m} \cdot T_{\text{diff}}^{(m)}$
(see Theorem \ref{prop:complexity}).
Then, the speedup is computed as:
\begin{equation}
\text{Speedup} = 1 - \frac{T_{\text{total}}}{T_{\text{diff}}^{(d)}} 
= 1 - \left( \frac{m}{d} + \frac{m}{d}  \frac{T_{\text{CS}}^{(m,d)}}{T_{\text{diff}}^{(m)}} \right).
\end{equation}

\subsection{Results on MNIST}

The MNIST dataset \cite{Lecun02} consists of images of handwritten digits, 
and on average, over $80 \%$ of the pixels in each image have intensity values equal or very close to zero.
As mentioned, we resize the images to the ambient resolutions $d \in \{32 \times 32, 40 \times 40, 48 \times 48\}$,
and fix the compressed dimension at $m = 28 \times 28$. 
We train a VP model in $\mathbb{R}^m$ for diffusion inference,
and decode the generated sample to $\mathbb{R}^d$ by FISTA.

\quad Table \ref{tab:mnist-time} reports per-image wall-clock for
(i) diffusion inference in $\mathbb{R}^m$ and 
(ii) FISTA recovery in $\mathbb{R}^d$,
along with the speedup.
As the ambient dimension $d$ increases (or the retention $m/d$ drops),
the diffusion inference time in the latent space $\mathbb{R}^m$ stays roughly constant 
with the recovery adding a small overhead,
while the diffusion inference in the ambient space grows with $d$.
This leads to increasing net speedups (from $4.39\%$ up to $61.13\%$).
\begin{table}[h]
    \centering
    \scalebox{0.8}{%
    \begin{tabular}{cccccc}
        \toprule
        Compression & Original Dim. & Original Dim. Inference Time & Low Dim. Inference Time & Recovery Time & Speedup \\
        \midrule
        76\% & 1024 $\to$ 784 & 0.4463s / pic & 0.3417s / pic & 0.0852s / pic & \textbf{4.39\%} \\
        49\% & 1600 $\to$ 784 & 1.1103s / pic & 0.5441s / pic & 0.0741s / pic & \textbf{44.32\%} \\
        34\% & 2304 $\to$ 784 & 1.5987s / pic & 0.5440s / pic & 0.0774s / pic & \textbf{61.13\%} \\
        \bottomrule
    \end{tabular}
    }
    \vspace{0.5em}
    \caption{Comparison of generation time on MNIST}
    \label{tab:mnist-time}
\end{table}

\quad Figure \ref{fig:mnist-compression} illustrates CSDM generations
at each compression level. 
With low compression/high retention ($76\%$),
digits are crisp and legible with thin strokes largely intact.
But with high compression/low retention ($34\%$), 
we observe a higher background grain and occasional breaks in tight curves,
with loop digits ($0/6/8$) and multi-segment (5) the first to degrade.
Nevertheless, class identity remains visible in most samples,
with the strong speedup at this compression.
Overall, CSDM achieves substantial wall-clock savings while preserving digit identity over a wide range of compression;
artifacts concentrate in thin/curved strokes at aggressive compression.
\begin{figure}[h]
  \centering
  \begin{subfigure}[t]{0.32\linewidth}
    \centering
    \includegraphics[width=\linewidth]{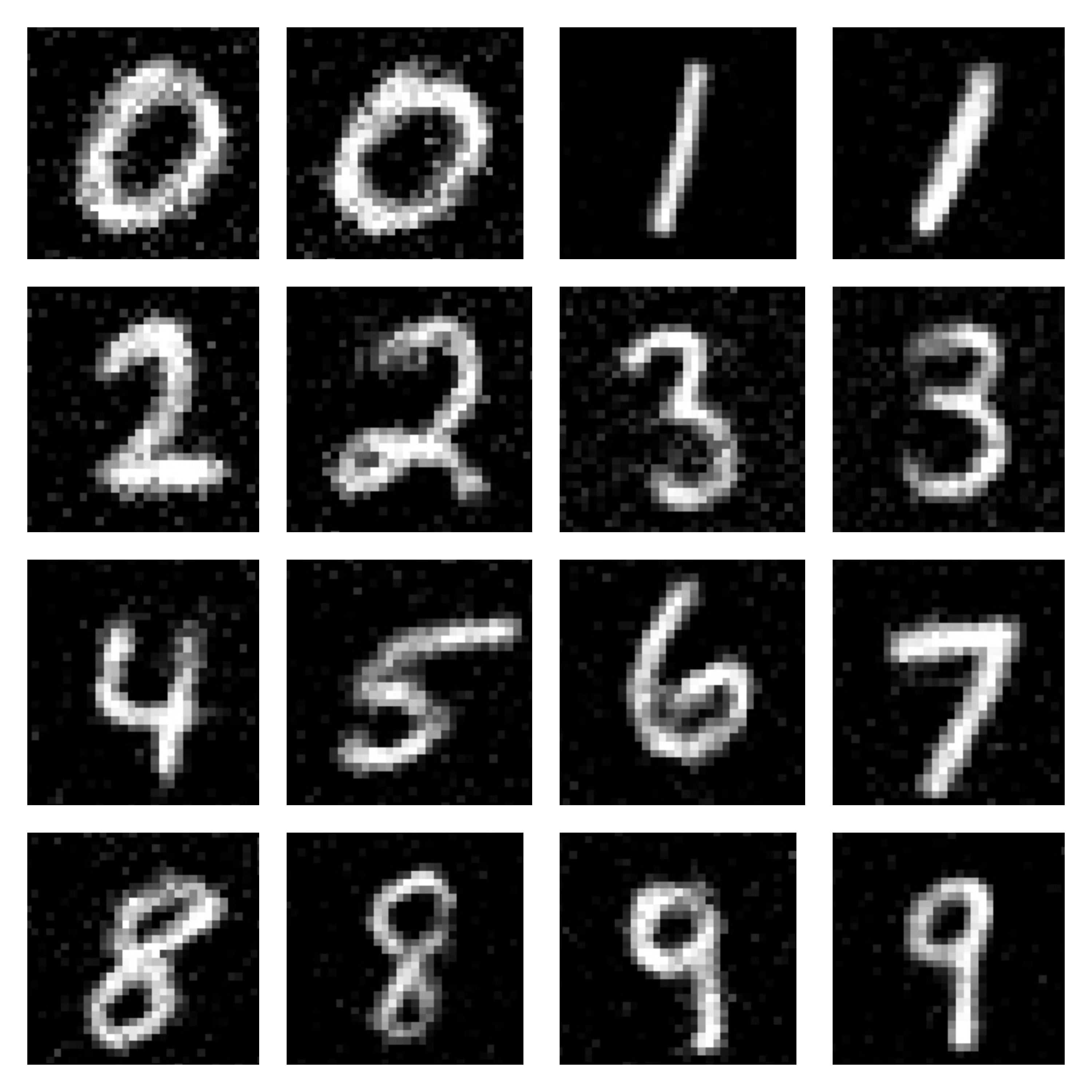}
    \subcaption{76\% dimensions retained}
  \end{subfigure}\hfill
  \begin{subfigure}[t]{0.32\linewidth}
    \centering
    \includegraphics[width=\linewidth]{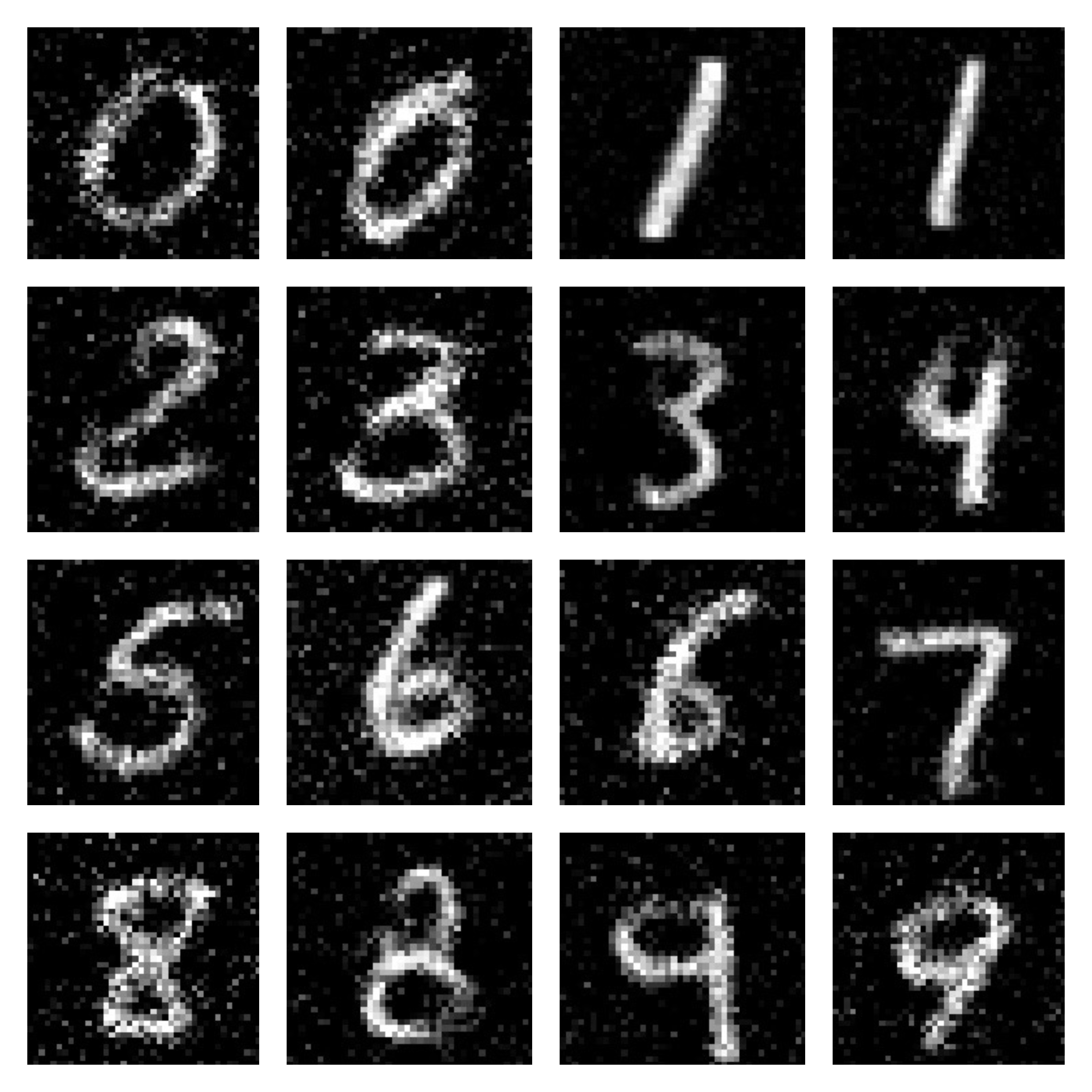}
    \subcaption{49\% dimensions retained}
  \end{subfigure}\hfill
  \begin{subfigure}[t]{0.32\linewidth}
    \centering
    \includegraphics[width=\linewidth]{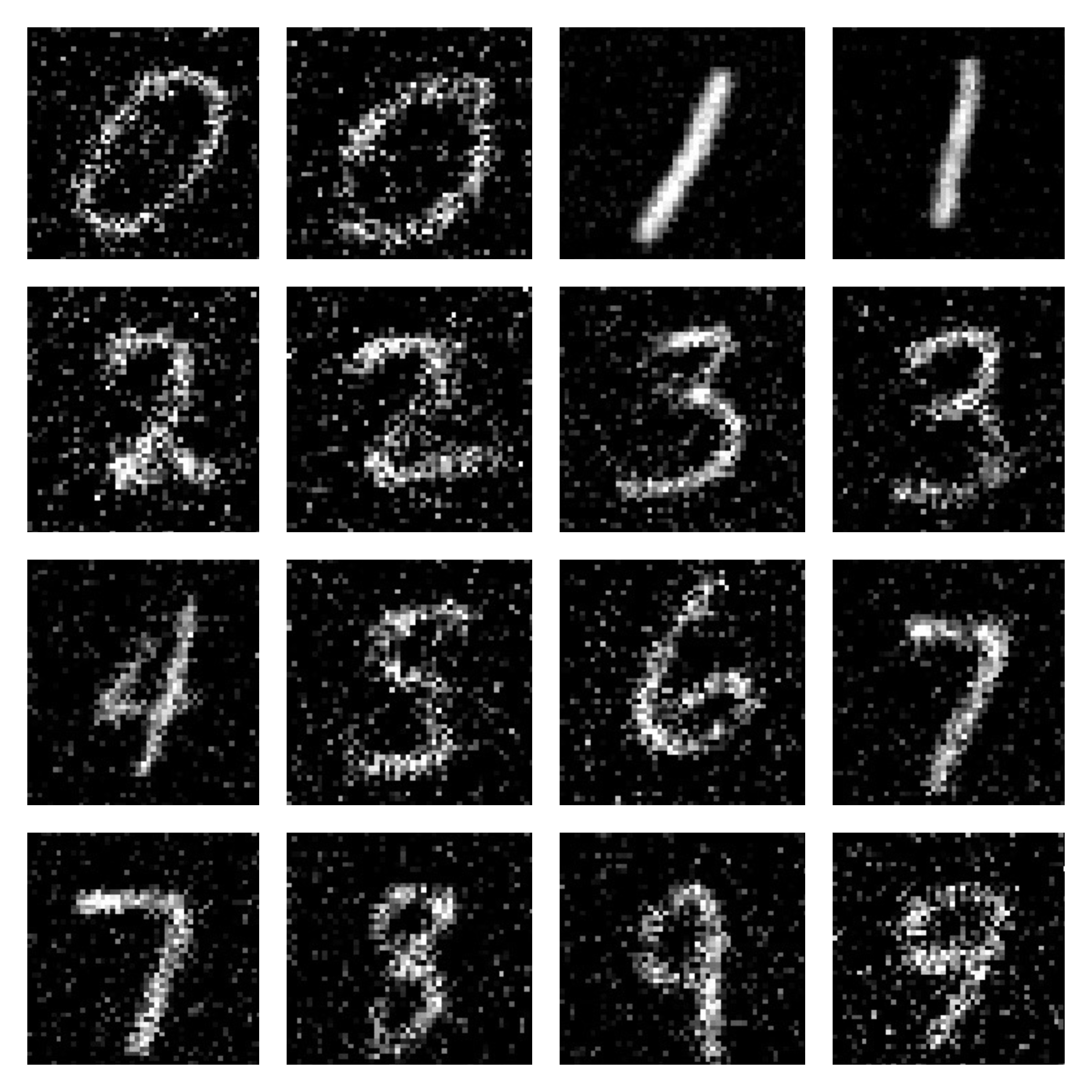}
    \subcaption{34\% dimensions retained}
  \end{subfigure}

  \caption{MNIST generations at three compression levels.}
  \label{fig:mnist-compression}
\end{figure}

\subsection{Results on OCTMNIST}

OCTMNIST contains retinal OCT B-scans from the MedMNIST collection \cite{YS23}.
Unlike handwritten digits, medical images are generally less sparse.
However, OCT exhibits banded anatomy: 
most diagnostic content concentrates in a narrow horizontal band (retinal layers) near the upper or middle part of the frame;
while large regions, especially the lower half, are near-zero background (see Figure \ref{fig:oct} for illustration).
This induces substantial spatial sparsity, which is around $65–70\%$ near-zero pixels at native resolution.
\begin{figure}[h]
    \centering
    \includegraphics[width=0.35\linewidth]{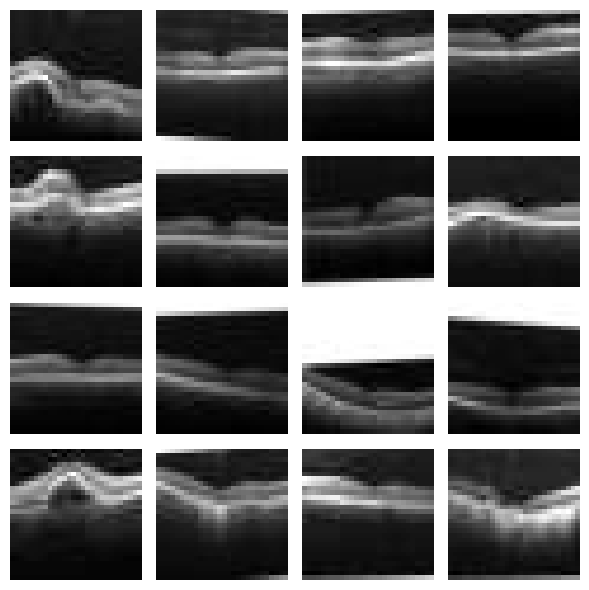}
    \caption{Original OCTMNIST samples}
    \label{fig:oct}
\end{figure}

\quad We follow the same setup as in MNIST.
Table \ref{tab:oct-time} reports per-image wall-clock for 
diffusion inference in $\mathbb{R}^m$, 
FISTA recovery in $\mathbb{R}^d$,
along with the speedup.
The result is similar to MNIST:
as the dimension $d$ increases,
the diffusion inference time in the latent space stays roughly constant  
with the recovery adding a small overhead, 
which yields larger net time savings.
\begin{table}[h]
    \centering
    \scalebox{0.8}{%
    \begin{tabular}{cccccc}
        \toprule
        Compression & Original Dim. & Original Dim. Inference Time & Low Dim. Inference Time & Recovery Time & Speedup \\
        \midrule
        76\% & 1024 $\to$ 784 & 1.6465s / pic & 1.2606s / pic & 0.1519s / pic & \textbf{4.99\%} \\
        49\% & 1600 $\to$ 784 & 1.9243s / pic & 0.9429s / pic & 0.1541s / pic & \textbf{42.99\%} \\
        34\% & 2034 $\to$ 784 & 2.8735s / pic & 1.1076s / pic & 0.1556s / pic & \textbf{56.04\%} \\
        \bottomrule
    \end{tabular}
    }
    \vspace{0.5em}
    \caption{Comparison of generation time on OCTMNIST}
    \label{tab:oct-time}
\end{table}

\quad Figure \ref{fig:oct-compression} shows CSDM generations at different compression levels.
With low compression/high retention ($76\%$),
the retinal band is continuous and well localized; intra-band texture appears with mild grain, and the background remains largely quiescent. Layer transitions are visible, with only light speckle around boundaries.
With high compression/low retention ($34\%$),
the band stays recognizable and contiguous, yet shows higher intra-band speckle and occasional softening at sharp transitions; background grain is more pronounced.
Overall, CSDM preserves the banded retinal anatomy,
while delivering substantial wall-clock savings. 
Artifacts concentrate as mild speckle and slight softening within the band 
at aggressive compression, but the decision-relevant structure (e.g., band continuity and localization) remains clear across settings.
\begin{figure}[h]
  \centering
  \begin{subfigure}[t]{0.32\linewidth}
    \centering
    \includegraphics[width=\linewidth]{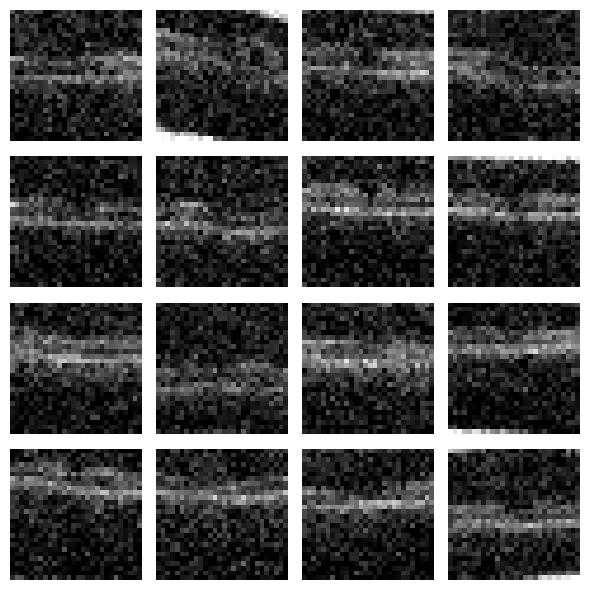}
    \subcaption{77\% dimensions retained}
  \end{subfigure}\hfill
  \begin{subfigure}[t]{0.32\linewidth}
    \centering
    \includegraphics[width=\linewidth]{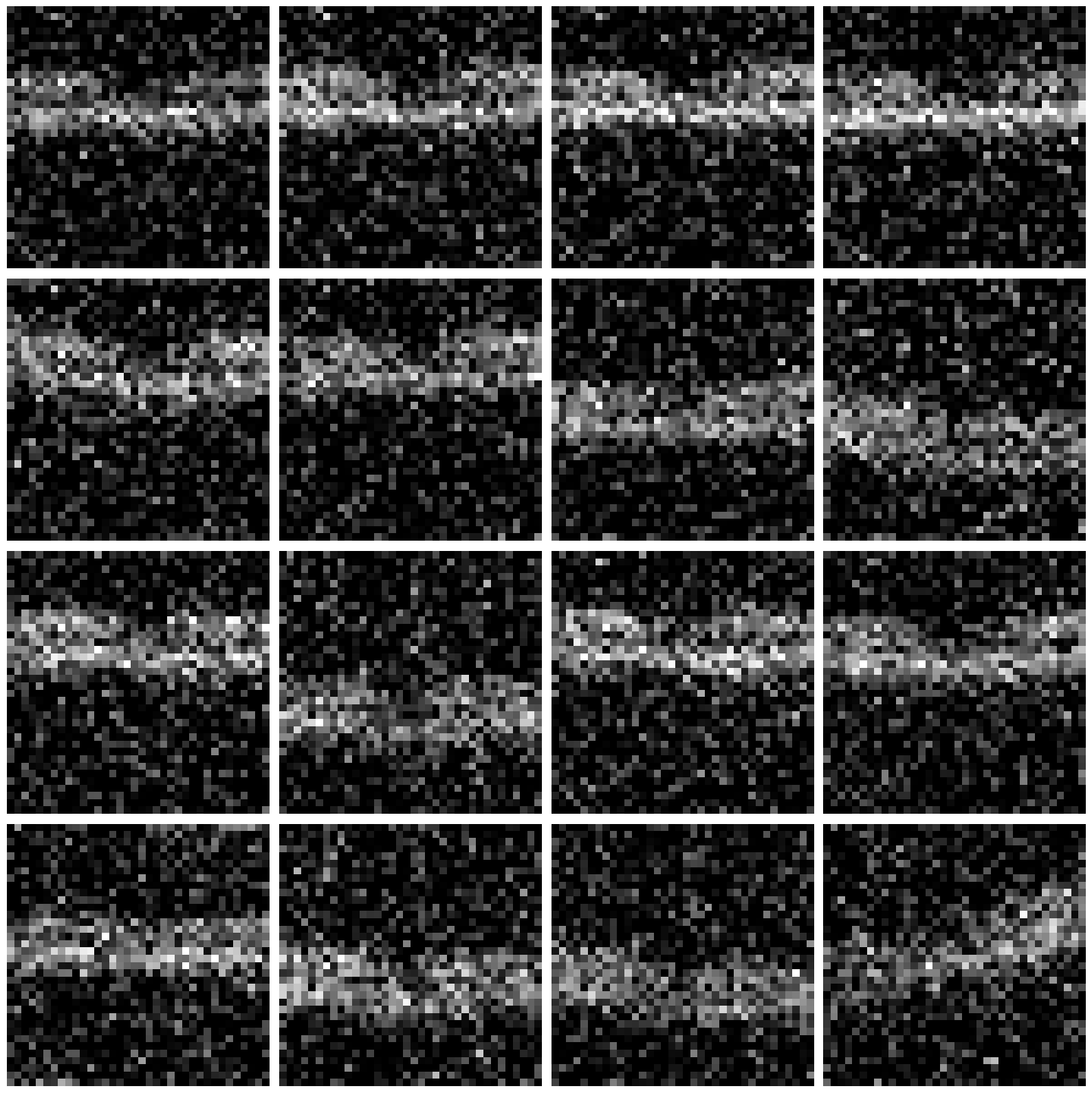}
    \subcaption{49\% dimensions retained}
  \end{subfigure}\hfill
  \begin{subfigure}[t]{0.32\linewidth}
    \centering
    \includegraphics[width=\linewidth]{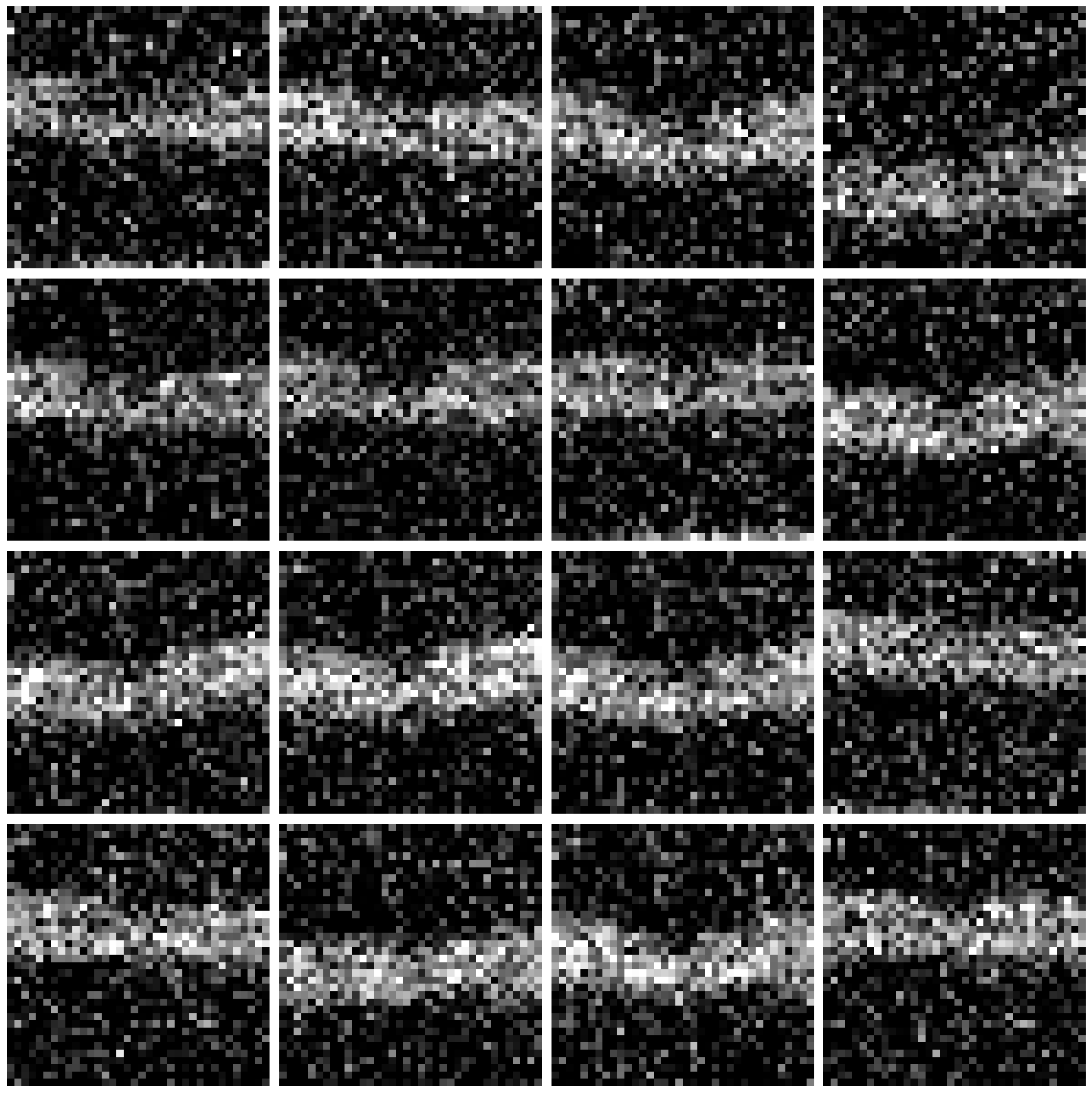}
    \subcaption{34\% dimensions retained}
  \end{subfigure}

  \caption{OCTMNIST generations at three compression levels.}
  \label{fig:oct-compression}
\end{figure}

\subsection{Results on ERA5 Reanalysis}

ERA5 Reanalysis dataset is provided by ECMWF on the large-scale precipitation fraction (LSPF) field (see Figure \ref{fig:original-oct} for illustration).
In contrast with the previous two subsections,
each snapshot is resized to a fixed ambient resolution of $80 \times 80$,
and then compressed to the retention levels $64\%$ ($64 \times 64$), $49\%$ ($56 \times 56$), and $36\%$ ($48 \times 48$).
\begin{figure}[h]
    \centering
    \includegraphics[width=0.35\linewidth]{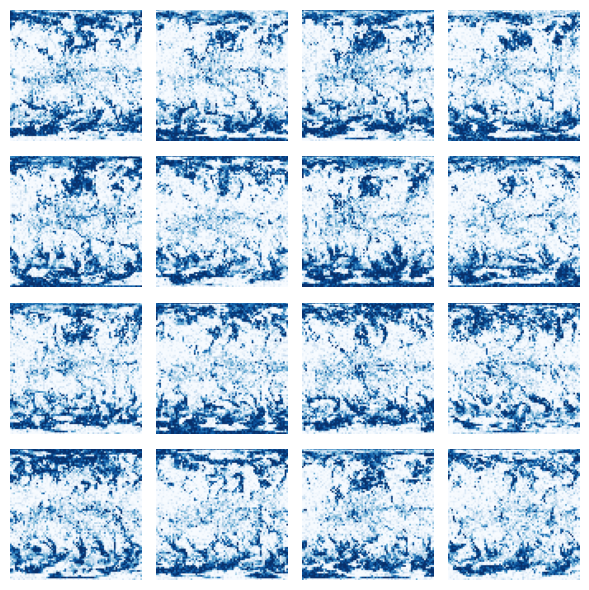}
    \caption{Original LSPF samples in Year 2023}
    \label{fig:original-oct}
\end{figure}

\quad Table \ref{tab:lspf-time} reports the per-sample wall-clock for diffusion inference in $\mathbb{R}^m$, FISTA recovery in $\mathbb{R}^d$, along with the speedup.
As the level of compression increases, 
the diffusion inference time in the latent space shortens significantly 
with the recovery remaining lightweight;
the net speedup increases steadily from $4.22\%$ to $59.31\%$.  
\begin{table}[h]
    \centering
    \resizebox{\textwidth}{!}{%
    \begin{tabular}{cccccc}
        \toprule
        Compression & Original Dim. & Original Dim. Inference Time & Low Dim. Inference Time & Recovery Time & Speedup \\
        \midrule
        64\% & 6400 $\to$ 4096 & 13.7545s / pic & 8.8029s / pic & 4.3712s / pic & \textbf{4.22\%} \\
        49\% & 6400 $\to$ 3136 & 12.8296s / pic & 6.2865s / pic & 1.7669s / pic & \textbf{37.23\%} \\
        36\% & 6400 $\to$ 2304 & 12.2049s / pic & 4.3938s / pic & 0.5721s / pic & \textbf{59.31\%} \\
        \bottomrule
    \end{tabular}
    }
    \caption{Comparison of generation time on LSPF}
    \label{tab:lspf-time}
\end{table}

\quad Figure \ref{fig:lspf-compression} illustrates CSDM generations at different compression levels.
With low compression/high retention ($64\%$), 
the generations are nearly indistinguishable from the full-resolution fields.
Fine-scale precipitation patterns are well preserved, with only minor smoothing in localized regions.
With high compression/low retention ($36\%$),
large-scale structures are still visible, but finer details are partially lost, and small patches may merge or vanish. 
Overall, CSDM achieves significant wall-clock savings,
while retaining essential spatial patterns on a complex and low-sparsity climate dataset. 
As the level of compression increases, artifacts manifest primarily in the loss of local variability, 
but the large-scale precipitation dynamics remain intact for downstream geophysical and risk analysis.
\begin{figure}[h]
  \centering
  \begin{subfigure}[t]{0.32\linewidth}
    \centering
    \includegraphics[width=\linewidth]{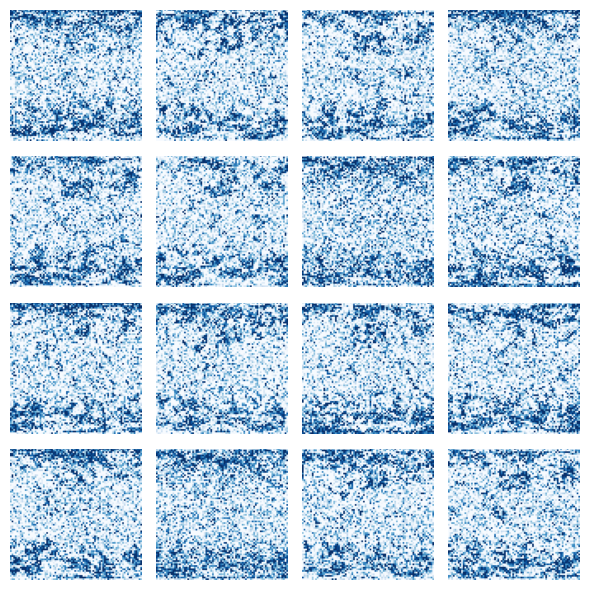}
    \subcaption{64\% dimensions retained}
  \end{subfigure}\hfill
  \begin{subfigure}[t]{0.32\linewidth}
    \centering
    \includegraphics[width=\linewidth]{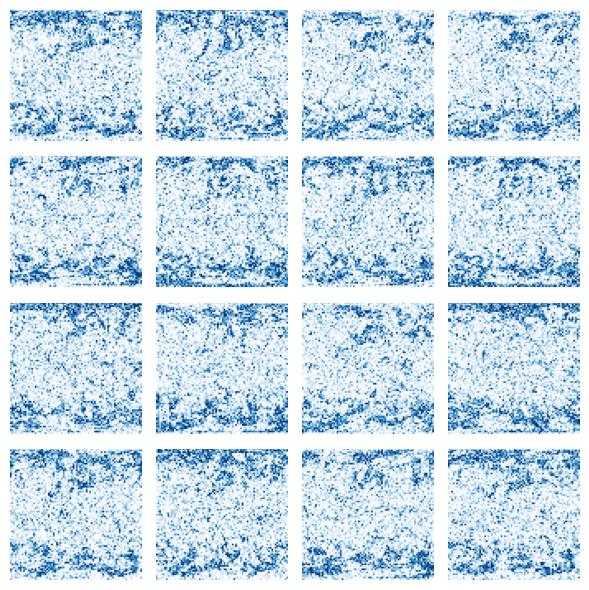}
    \subcaption{49\% dimensions retained}
  \end{subfigure}\hfill
  \begin{subfigure}[t]{0.32\linewidth}
    \centering
    \includegraphics[width=\linewidth]{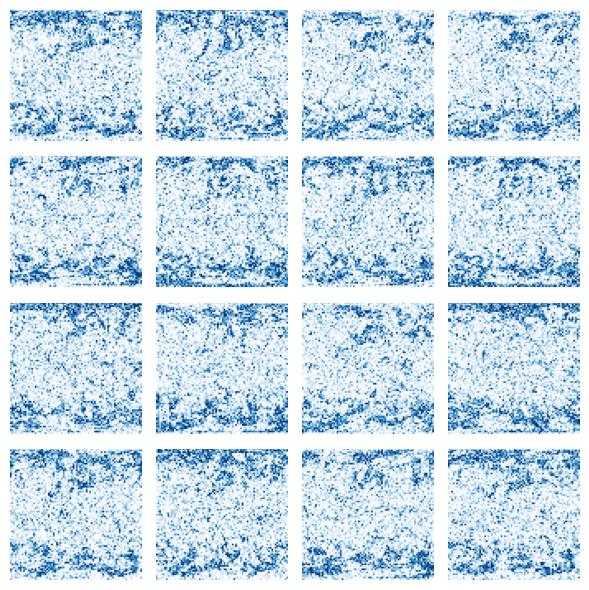}
    \subcaption{36\% dimensions retained}
  \end{subfigure}

  \caption{LSPF generations at three compression levels.}
  \label{fig:lspf-compression}
\end{figure}

\section{Numerical experiments on time series data}
\label{sc5}

\quad In this section, we further explore the idea of integrating diffusion models with dimension reduction techniques in the context of financial time series.
Previous works \cite{AB25, CX25} applied diffusion models for portfolio optimization.
Our focus here is on stress testing using a data-driven approach via diffusion generative models.
Specifically, we use principal component analysis (PCA) to find the most significant variance directions of macroeconomic factors.
We then train a diffusion model in the principle component (PC) space 
to generate synthetic PC data.
These generated PC data can be viewed as ``informative" factors,
which can be subsequently used for portfolio backtesting and stress testing via regression analysis.

\quad We train a diffusion model in a low-dimensional macro-factor space:
the first $6$ PCs computed from 126 FRED-MD factors \cite{mccracken2015fredmd},
corresponding to over $90\%$ of cumulative explained variance.
We then generate synthetic PC paths, and map them to the equity space
(AAPL, AMZN, COST, CVX, GOOGL, JPM, KO, MCD, NVDA, UNH)
for portfolio management (Section \ref{sc51}) and stress testing (Section \ref{sc52}).

\smallskip
{\em Value at Risk (VaR)}:
For a portfolio return $R$, the $\alpha$-quantile $Q_\alpha(R)$ induces the (left-tail) VaR as
\[
\mathrm{VaR}_{\alpha} = -\,Q_{\alpha}(R).
\]
The tables below report the quantiles of returns,
with the $1\%$, $5\%$, $10\%$ and $25\%$ rows corresponding  to $\mathrm{VaR}_{99\%}$, $\mathrm{VaR}_{95\%}$, $\mathrm{VaR}_{90\%}$ and $\mathrm{VaR}_{75\%}$ (after the sign change).
 
\subsection{Unconditional Portfolio Management}
\label{sc51}

We evaluate $6$-month cumulative log-returns under three portfolio constructions:
(i) Equal-Weight portfolio, 
(ii) Markowitz global minimum variance portfolio (GMVP), and
(iii) Risk-Parity portfolio.
If the low-dimensional PCs retain the key risk directions, 
then the generated data are expected to reproduce distributional properties (e.g., center, dispersion and tails).

\medskip
(i) {\em Equal-Weight Portfolio}: 
Figure \ref{fig:arith-hist} shows the histograms of real and generated 6-month cumulative log-returns,
and Table \ref{tab:arith-stats} provides the summary statistics. 
\begin{table}[h]
  \centering
  \renewcommand{\arraystretch}{1} 
  \begin{tabular}{lrr}
    \toprule
    \textbf{Statistics} & \textbf{Real 6M} & \textbf{Generated 6M} \\
    \midrule
    Mean           & 9.43\%   & 6.78\% \\
    Median         & 10.38\%  & 7.62\% \\
    Std Dev        & 8.83\%   & 9.13\% \\
    1\% Quantile   & -12.65\% & -16.35\% \\
    5\% Quantile   & -6.64\%  & -10.05\% \\
    10\% Quantile  & -2.67\%  & -5.25\% \\
    25\% Quantile  & 4.17\%   & 1.95\% \\
    \bottomrule
  \end{tabular}
  \caption{Equal-weight portfolio: real vs.\ generated 6M log-return statistics.}
  \label{tab:arith-stats}
\end{table}

The generated portfolio distribution aligns with the real one in location and scale,
while showing a heavier left tail.
The VaR errors are the largest for the equal-weight portfolio
(e.g., the 5\% quantile differs by $\approx\!3.41$pp),
indicating more mass in the synthetic left tail with no risk-adjusted weights.

\medskip
(ii) {\em Markowitz GMVP}:
Table \ref{tab:gmvp-weights} provides the GMVP weights (with short-selling not allowed),
and Figure \ref{fig:gmvp-hist} shows the histograms of real and generated log-returns.
The summary statistics (Table \ref{tab:gmvp-stats})
and the efficient frontiers (Figure \ref{fig:eff-front})
show that the real and generated portfolios have close mean and volatilities,
but moderate tail differences.
Also note that GMVP tails are much closer than those in the equal-weight case
(e.g., $|\Delta Q_{1\%}| \approx 0.62$pp and $|\Delta Q_{5\%}| \approx 1.03$pp),
suggesting that the PC diffusion generations preserve the covariance structure 
relevant to volatility minimization.
\begin{table}[h]
  \centering
  \renewcommand{\arraystretch}{1}
  \begin{tabular}{lrrrrr}
    \toprule
    \textbf{Source} & \textbf{AAPL} & \textbf{AMZN} & \textbf{COST} & \textbf{CVX} & \textbf{GOOGL} \\
    \midrule
    Real (\%)      & 0.00 & 4.74 & 17.83 & 1.61 & 9.47 \\
    Generated (\%) & 0.00 & 1.35 & 28.85 & 2.47 & 6.05 \\
    \bottomrule
  \end{tabular}

  \vspace{0.5em}

  \begin{tabular}{lrrrrr}
    \toprule
    \textbf{Source} & \textbf{JPM} & \textbf{KO} & \textbf{MCD} & \textbf{NVDA} & \textbf{UNH} \\
    \midrule
    Real (\%)      & 1.45 & 27.99 & 25.34 & 0.00 & 11.56 \\
    Generated (\%) & 0.00 & 39.40 & 13.39 & 0.00 & 10.96 \\
    \bottomrule
  \end{tabular}

  \caption{GMVP portfolio weights comparison (real vs.\ generated).}
  \label{tab:gmvp-weights}
\end{table}

\begin{table}[h]
  \centering
  \renewcommand{\arraystretch}{1}
  \begin{tabular}{lrr}
    \toprule
    \textbf{Statistics} & \textbf{Real GMVP} & \textbf{Predicted GMVP} \\
    \midrule
    Mean           & 7.26\%  & 7.28\%  \\
    Median         & 7.63\%  & 7.01\%  \\
    Std Dev        & 6.19\%  & 6.42\%  \\
    1\% Quantile   & -8.52\% & -7.90\% \\
    5\% Quantile   & -2.78\% & -3.81\% \\
    10\% Quantile  & -0.26\% & -0.73\% \\
    25\% Quantile  & 3.65\%  & 3.52\%  \\
    \bottomrule
  \end{tabular}
  \caption{GMVP: real vs.\ generated 6M log-return statistics.}
  \label{tab:gmvp-stats}
\end{table}

(iii) {\em Risky-Parity Portfolio}:
Table \ref{tab:rp-weights} reports the risk-parity weights,
which are similar in both settings.
Figure \ref{fig:rp-hist} provides the histograms of real and generated log-returns. 
The summary statistics (Table \ref{tab:rp-stats}) shows that 
mean and volatilities match,
and the left-tail quantiles differ by only $0.01$--$0.04$pp,
indicating that tail risk is effectively captured when portfolios are constructed from risk-balanced exposures.

\begin{table}[htbp]
  \centering
  \renewcommand{\arraystretch}{1}

  \begin{tabular}{lrrrrr}
    \toprule
    \textbf{Source} & \textbf{AAPL} & \textbf{AMZN} & \textbf{COST} & \textbf{CVX} & \textbf{GOOGL} \\
    \midrule
    Real (\%)      & 7.90 & 8.37 & 11.56 & 9.05 & 9.23 \\
    Generated (\%) & 7.31 & 8.62 & 12.29 & 9.04 & 9.43 \\
    \bottomrule
  \end{tabular}

  \vspace{0.5em}

  \begin{tabular}{lrrrrr}
    \toprule
    \textbf{Source} & \textbf{JPM} & \textbf{KO} & \textbf{MCD} & \textbf{NVDA} & \textbf{UNH} \\
    \midrule
    Real (\%)      & 8.23 & 14.67 & 14.27 & 5.42 & 11.30 \\
    Generated (\%) & 8.50 & 14.71 & 13.50 & 5.14 & 11.46 \\
    \bottomrule
  \end{tabular}

  \caption{Risk-Parity portfolio weights comparison (real vs.\ generated).}
  \label{tab:rp-weights}
\end{table}

\begin{table}[h]
  \centering
  \renewcommand{\arraystretch}{1}
  \begin{tabular}{lrr}
    \toprule
    \textbf{Statistics} & \textbf{Real RP} & \textbf{Predicted RP} \\
    \midrule
    Mean            & 8.55\%  & 8.54\%  \\
    Median          & 8.85\%  & 8.71\%  \\
    Std Dev         & 7.43\%  & 7.41\%  \\
    1\% Quantile    & -11.33\% & -11.32\% \\
    5\% Quantile    & -3.47\%  & -3.51\%  \\
    10\% Quantile   & -1.48\%  & -1.46\%  \\
    25\% Quantile   & 4.00\%   & 4.17\%   \\
    \bottomrule
  \end{tabular}
  \caption{Risk-Parity: real vs.\ generated 6M log-return statistics.}
  \label{tab:rp-stats}
\end{table}

\begin{figure}[h]
  \centering
  \begin{minipage}[t]{0.45\textwidth}
    \centering
    \includegraphics[width=\textwidth]{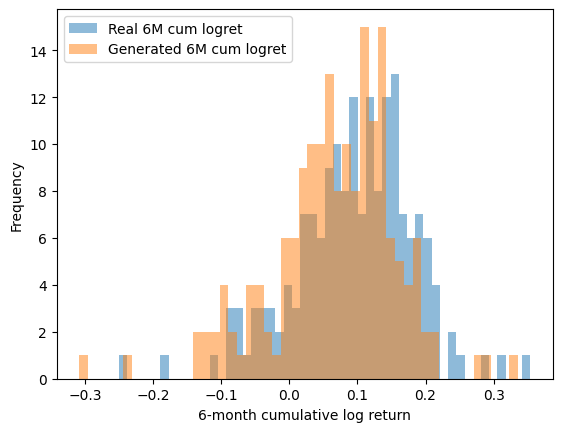}
    \caption{Equal-weight portfolio comparison.}
    \label{fig:arith-hist}
  \end{minipage}\hfill
  \begin{minipage}[t]{0.45\textwidth}
    \centering
    \includegraphics[width=\textwidth]{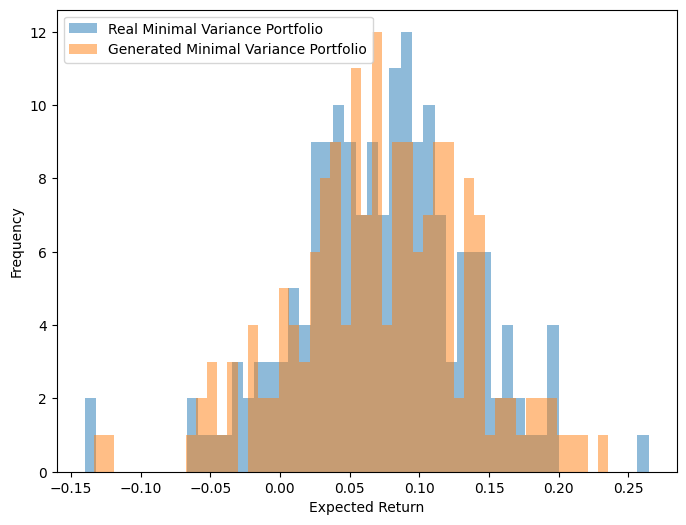}
    \caption{GMVP comparison.}
    \label{fig:gmvp-hist}
  \end{minipage}
  
  \vspace{0.5cm} 

  \begin{minipage}[t]{0.45\textwidth}
    \centering
    \includegraphics[width=\textwidth]{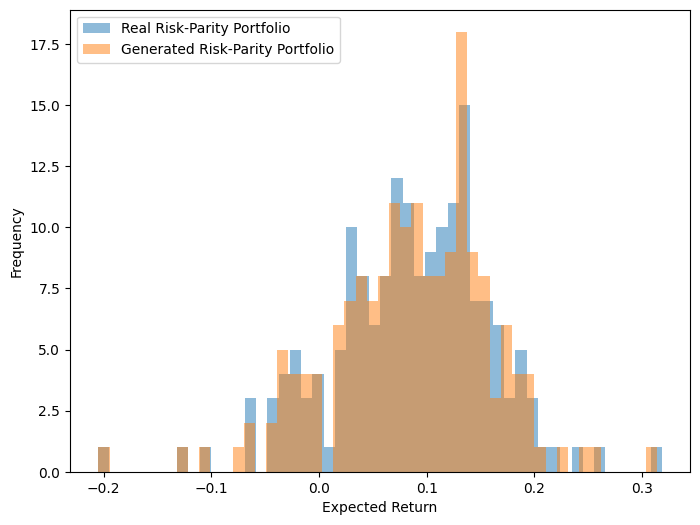}
    \caption{Risk-parity comparison.}
    \label{fig:rp-hist}
  \end{minipage}\hfill
  \begin{minipage}[t]{0.45\textwidth}
    \centering
    \includegraphics[width=\textwidth]{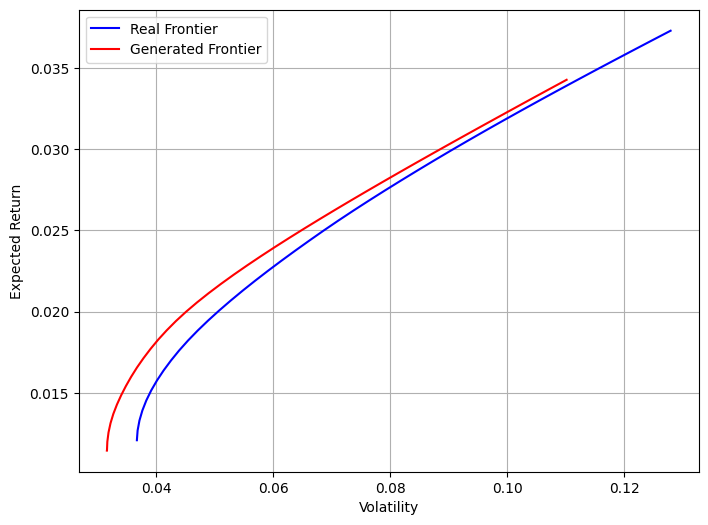}
    \caption{Efficient frontiers comparison.}
    \label{fig:eff-front}
  \end{minipage}
\end{figure}

\quad Overall, training a diffusion model in a low-dimensional macro-PC space 
has proven reliable in reproducing the real data distribution for backtesting
across equal-weight, GMVP, and risk-parity portfolios,
and crucially,
the left-tail quantiles are close to their empirical counterparts.
Weight patterns are consistent for GMVP and risk-parity, 
indicating that the low-dimensional PCs retain the covariance
structure that drives risk-aware portfolio construction.

\subsection{Standard scenario analysis (SSA)}
\label{sc52}

{\em Standard scenario analysis} (SSA) for financial portfolios is designed to estimate a portfolio’s return when some subsets of risk factors are subjected to stress \cite{HR20}. 
Here we follow the presentation of \cite{baker2025datadrivendynamicfactormodeling}.

\quad Under the standard multi-factor model for stock returns, 
if $X_t \in \mathbb{R}^d$ is the $d$-dimensional vector of common factors,
we define $\mathcal{S}$ to be the set containing the indices of the factors in a scenario (i.e., the factors that we intend to stress), 
and hence $\mathcal{S}^c$ is the index set of those factors that we leave un-stressed. 
SSA stresses the components of $X_{S,t}$
according to a given scenario (e.g., $+20$ on the S\&P index, $-10$ on the CPI index, and $+0.1$ on the US Dollar/Euro exchange rate), 
and keeps the components in $X_{S^C,t}$ unchanged  (i.e., equal to their current value).
The new portfolio P\&L, or overall return $V_t$ is then 
computed with $Y_t$ determined by the scenario and the multi-factor model.
To be more precise,
\begin{enumerate}[itemsep = 3 pt]
  \item 
  Let $\Delta X_{S,t+1} = X_{S,t+1} - X_{S,t}$ denote the $t+1$ scenario stress vector $\in \mathbb{R}^{|\mathcal{S}|}$.
  \item 
  Compute the SSA factor change vector:
  \[
  \Delta X^{SSA}_{i,t+1} =
  \begin{cases}
    \Delta X_{S,t+1}, & \text{for } i \in \mathcal{S}, \\
    0, & \text{for } i \in \mathcal{S}^c.
  \end{cases}
  \]
  \item 
  Obtain the predicted factor vector under SSA:
  $X^{SSA}_{t+1} = X_t + \Delta X^{SSA}_{t+1}$.
  \item 
  From the fitted neural net $f_{NN}$, predict the post-stress asset returns:
  $Y^{SSA-\text{stress}}_{t+1} = f_{NN}(X^{SSA}_{t+1})$.
\end{enumerate}

\quad We summarize the SSA procedure in the following algorithm, which we will later feed into a historical rolling window backtest. 
Let $s$ be the size of the rolling window.
Note that $f_{NN}$ is treated as an input in this algorithmic format. 
Denote by $x_{t-s:t} \in \mathbb{R}^{s \times d}$ the matrix of common factors from time $t-s$ up to $t$, and by $y_{t-s:t} \in \mathbb{R}^{s \times n}$ the matrix of asset returns from time $t-s$ up to $t$.
We take $x_{\text{actual},S,t+1}$ to be the realized historical scenario.
\begin{tcolorbox}
\textbf{Input:} $x_{t-s:t}$ (common factors in rolling window), 
$x_{t+1}$ (future factors), 
$f_{NN}$ (neural net fitted on the whole time series) \\[0.5em]
\textbf{Output:} $\hat{V}_{t+1,SSA}$ (portfolio return under SSA)

\begin{enumerate}
  \item Set $x_{S,t+1} \gets x_{\text{actual},S,t+1}$ and 
        $x_{S^C,t+1} \gets x_{S^C,t}$ to form $x^{SSA}_{t+1}$.
  \item Compute $y_{t+1,SSA} \gets f_{NN}(x^{SSA}_{t+1}$).
  \item Compute portfolio weights $w \gets 1/N$ (or Markowitz weights, or Risk-Parity).
  \item Compute portfolio return $\hat{V}_{t+1,SSA} \gets w^\top y_{t+1,SSA}$.
\end{enumerate}
\end{tcolorbox}

\quad In our experiments, 
the goal is to evaluate the performance of the generated data against real data 
in the context of financial stress testing.
For the real data, we train a neural network model $f_{NN}^{\text{macro}}$, where the input consists of 126 macroeconomic factors. 
In this setting, the variable $x^{SSA}$ in the aforementioned algorithm corresponds directly to these macro factors.
The generated data are constructed from the first six principal components (PCs). 
To process these synthetic features, we employ a different neural net $f_{NN}^{\text{PC}}$, which takes as input the time series of six PCs, 
and outputs the corresponding stock prices. 
In this case, $x^{SSA}$ are the values of the principal components instead of the original macroeconomic factors.

\quad We conduct SSA by stressing selected macro factors while holding all others fixed,
and then computing portfolio returns under three strategies:
Equal Weight, Markowitz GMVP, and Risk-Parity.
We compare {\em Real Data SSA} (macro factors fed to a neural net trained on all $126$ factors)
with {\em Generated Data SSA} (diffusion inference in the PC space).
Three scenarios are selected among combination of the four following factors: (1) Real Personal Income (RPI) from Group Output and Income; (2) All Employees: Service-Providing Industries (SRVPRD) from Group Labor Market; (3) New Orders for Consumer Goods (ACOGNO) from Group Orders and Inventories; (4) S\&P’s Common Stock Price Index: Composite (S\&P~500) from Group Stock Market. The scenarios are considered:
(1) RPI + SRVPRD; (2) S\&P~500 + ACOGNO; (3) all four (RPI, SRVPRD, S\&P~500, ACOGNO).
Figures \ref{fig:ssa-s1}--\ref{fig:ssa-s3} illustrate the empirical return distributions under the two data sources.
Across all strategies and scenarios,
the generated distributions exhibit strong alignment with the real data benchmarks
in terms of central tendency, dispersion, and tail behavior.
This is further confirmed by detailed summary statistics in Tables \ref{table:ssa-s1}--\ref{table:ssa-s3}.
Notably, under the GMVP strategy, the generated data reproduce the real mean return within a margin of less than 0.2\%, while the extreme quantiles (1\%, 5\%) also display high fidelity, 
suggesting accurate modeling of downside risks.
Similar consistency is observed for the risk-parity strategy, 
where the generated returns track both the scale and distributional shape of real data.

\begin{figure}
  \centering
  \includegraphics[width=1\textwidth]{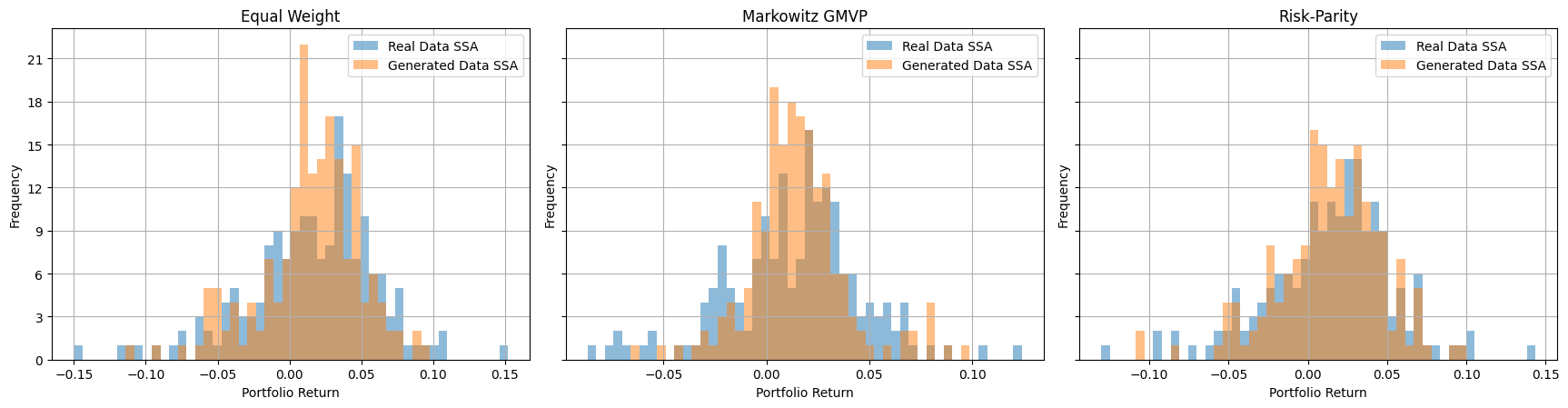}
  \caption{Real vs. Generated SSA (Scenario 1: RPI \& SRVPRD)}
  \label{fig:ssa-s1}
\end{figure}

\begin{table}[htbp]
\centering
\small
\resizebox{\textwidth}{!}{%
\begin{tabular}{llrrrrrrrr}
\toprule
Method & Source & mean & median & std & 1\% & 5\% & 10\% & 25\% \\
\midrule
Equal Weight   & Real Data SSA      & 0.015499 & 0.020229 & 0.045045 & -0.111204 & -0.061352 & -0.041970 & -0.008881 \\
Equal Weight   & Generated Data SSA & 0.014896 & 0.018448 & 0.034625 & -0.078333 & -0.053042 & -0.037523 &  0.001501 \\
Markowitz GMVP & Real Data SSA      & 0.013717 & 0.018585 & 0.033267 & -0.073650 & -0.038906 & -0.026783 & -0.004122 \\
Markowitz GMVP & Generated Data SSA & 0.015070 & 0.013561 & 0.023838 & -0.042804 & -0.021070 & -0.008217 &  0.003035 \\
Risk-Parity    & Real Data SSA      & 0.014376 & 0.018989 & 0.040037 & -0.096461 & -0.053468 & -0.037086 & -0.007106 \\
Risk-Parity    & Generated Data SSA & 0.013583 & 0.015629 & 0.033125 & -0.089767 & -0.047309 & -0.025966 & -0.003006 \\
\bottomrule
\end{tabular}
}
\caption{Summary statistics for Scenario 1 (RPI + SRVPRD).}
\label{table:ssa-s1}
\end{table}

\begin{figure}
  \centering
  \includegraphics[width=1\textwidth]{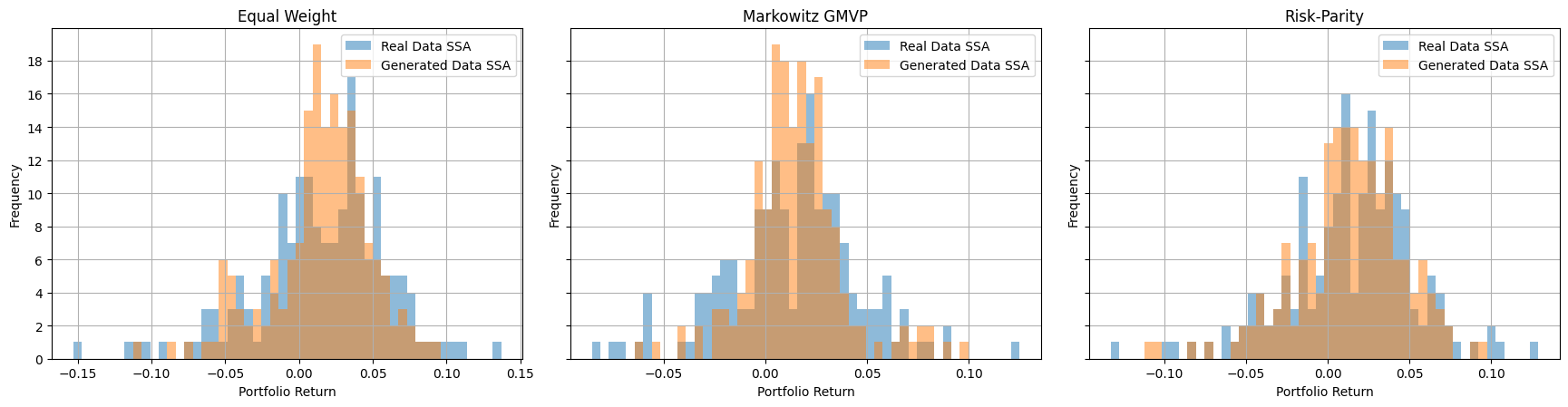}
  \caption{Real vs. Generated SSA (Scenario 2: S\&P~500 + ACOGNO).}
  \label{fig:ssa-s2}
\end{figure}

\begin{table}[htbp]
\centering
\small
\resizebox{\textwidth}{!}{%
\begin{tabular}{llrrrrrrrr}
\toprule
Method & Source & mean & median & std & 1\% & 5\% & 10\% & 25\% \\
\midrule
Equal Weight   & Real Data SSA      & 0.015336 & 0.020786 & 0.044551 & -0.110215 & -0.061371 & -0.041356 & -0.008964 \\
Equal Weight   & Generated Data SSA & 0.014808 & 0.018311 & 0.033909 & -0.078462 & -0.049675 & -0.040138 &  0.001706 \\
Markowitz GMVP & Real Data SSA      & 0.013588 & 0.018453 & 0.032662 & -0.073197 & -0.039339 & -0.026637 & -0.004582 \\
Markowitz GMVP & Generated Data SSA & 0.015042 & 0.014135 & 0.023652 & -0.042127 & -0.021630 & -0.008526 &  0.003317 \\
Risk-Parity    & Real Data SSA      & 0.014182 & 0.018976 & 0.039558 & -0.095410 & -0.052677 & -0.036798 & -0.007377 \\
Risk-Parity    & Generated Data SSA & 0.013485 & 0.015349 & 0.032551 & -0.086335 & -0.040740 & -0.026172 & -0.000841 \\
\bottomrule
\end{tabular}
}
\caption{Summary statistics for Scenario 2 (S\&P~500 + ACOGNO).}
\label{table:ssa-s2}
\end{table}

\begin{figure}
  \centering
  \includegraphics[width=1\textwidth]{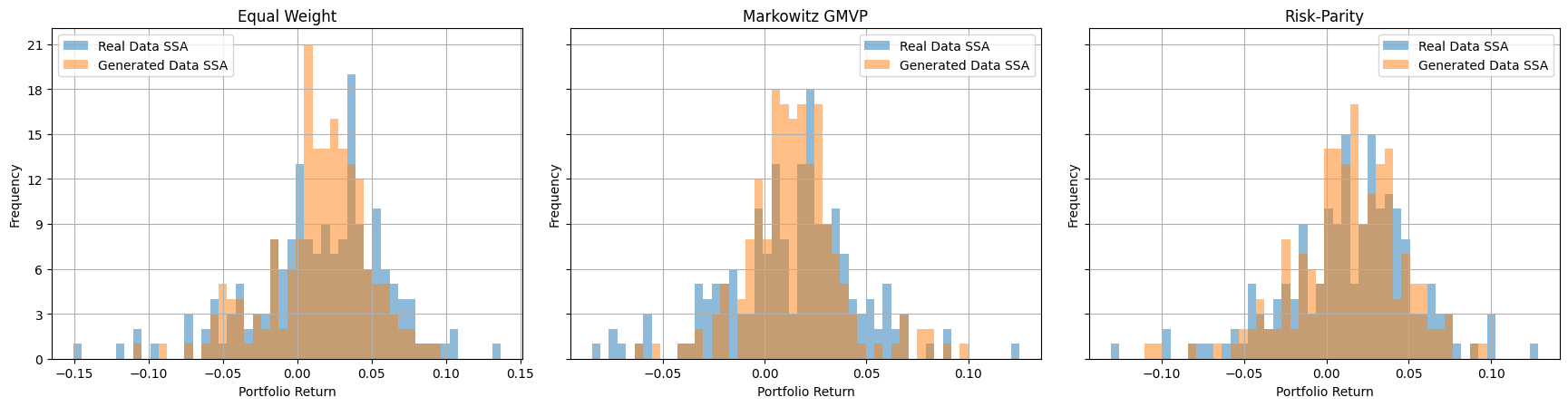}
  \caption{Real vs. Generated SSA (Scenario 3: RPI, SRVPRD, S\&P~500, ACOGNO).}
  \label{fig:ssa-s3}
\end{figure}

\begin{table}[htbp]
\centering
\small
\resizebox{\textwidth}{!}{%
\begin{tabular}{llrrrrrrrr}
\toprule
Method & Source & mean & median & std & 1\% & 5\% & 10\% & 25\% \\
\midrule
Equal Weight   & Real Data SSA      & 0.015177 & 0.018720 & 0.044467 & -0.110243 & -0.061746 & -0.041088 & -0.009005 \\
Equal Weight   & Generated Data SSA & 0.014738 & 0.018579 & 0.034219 & -0.079646 & -0.051498 & -0.040804 &  0.001805 \\
Markowitz GMVP & Real Data SSA      & 0.013350 & 0.017684 & 0.032525 & -0.070224 & -0.037711 & -0.026908 & -0.004720 \\
Markowitz GMVP & Generated Data SSA & 0.014791 & 0.014446 & 0.023785 & -0.043476 & -0.021829 & -0.009040 &  0.003186 \\
Risk-Parity    & Real Data SSA      & 0.014034 & 0.017923 & 0.039447 & -0.095408 & -0.052755 & -0.037053 & -0.007807 \\
Risk-Parity    & Generated Data SSA & 0.013469 & 0.014923 & 0.033141 & -0.087174 & -0.042830 & -0.027617 & -0.001268 \\
\bottomrule
\end{tabular}
}
\caption{Summary statistics for Scenario 3 (all four factors).}
\label{table:ssa-s3}
\end{table}

\section{Conclusion}
\label{sc6}

\quad In this study, we develop dimension reduction techniques to accelerate diffusion model inference for data generation.
The idea is to incorporate compressed sensing into diffusion sampling,
so as to facilitate the efficiency of both model training and inference.
Under suitable sparsity assumptions on data,
the proposed CSDM algorithm is proved to enjoy faster convergence,
and an optimal value for the latent space dimension is derived as a byproduct.
We also corroborate our theory with numerical experiments on various image data,
and financial time series for stress testing applications.

\quad There are several directions to extend this work.
First, an important problem is to derive sharper convergence rates of FISTA with explicit dimension dependence. 
This will allow us to obtain better complexity of the proposed CSDM algorithm.
Second, it will be interesting to integrate the proposed CSDM algorithm into conditional generation or guidance \cite{Dh21, HS21, Karras24, TX25},
and diffusion model alignment \cite{Black24, Fan23, ZC25}.
Finally, it will be desirable to further extend the study in Section~\ref{sc5} into a PCA + diffusion modeling framework.

\bigskip
{\bf Acknowledgment}:
Wenpin Tang is supported by
NSF grant DMS-2206038, the Columbia Innovation Hub grant, and the Tang Family Assistant
Professorship.
David D. Yao is part of a Columbia-CityU/HK collaborative
project that is supported by InnoHK Initiative, The Government of the HKSAR and the AIFT Lab.

\bibliographystyle{abbrv}
\bibliography{unique}
\end{document}